\documentclass[english]{article}
\usepackage[T1]{fontenc}
\usepackage[latin9]{inputenc}
\usepackage{verbatim}
\usepackage{units}
\usepackage{amsmath}
\usepackage{amsthm}
\usepackage{amssymb}
\usepackage{graphicx}

\makeatletter
\numberwithin{equation}{section}
\numberwithin{figure}{section}
\theoremstyle{plain}
\newtheorem{thm}{\protect\theoremname}
\theoremstyle{remark}
\newtheorem{rem}[thm]{\protect\remarkname}
\theoremstyle{plain}
\newtheorem{prop}[thm]{\protect\propositionname}
\theoremstyle{plain}
\newtheorem{lem}[thm]{\protect\lemmaname}
\theoremstyle{remark}
\newtheorem*{claim*}{\protect\claimname}
\theoremstyle{remark}
\newtheorem*{rem*}{\protect\remarkname}

  \PassOptionsToPackage{numbers}{natbib}

 \usepackage[preprint]{neurips_2020}

\usepackage[T1]{fontenc}    
\usepackage{hyperref}       
\usepackage{url}            
\usepackage{booktabs}       
\usepackage{amsfonts}       
\usepackage{nicefrac}       
\usepackage{microtype}      

\usepackage{wrapfig}

\title{Order and Chaos: NTK views on DNN Normalization, Checkerboard and Boundary Artifacts}

\author{%
  Arthur Jacot \\
  Ecole Polytechnique F\'ed\'erale de Lausanne \\
  \texttt{arthur.jacot@epfl.ch} \\
  \And
 Franck Gabriel \\
  Ecole Polytechnique F\'ed\'erale de Lausanne \\
  \texttt{franck.gabriel@epfl.ch} \\
  \And
 Fran\c{c}ois  Ged\\
  Ecole Polytechnique F\'ed\'erale de Lausanne \\
  \texttt{francois.ged@epfl.ch} \\
  \And
 Cl\'ement Hongler \\
  Ecole Polytechnique F\'ed\'erale de Lausanne \\
  \texttt{clement.hongler@epfl.ch} \\
}

\makeatother

\usepackage{babel}
\providecommand{\lemmaname}{Lemma}
\providecommand{\propositionname}{Proposition}
\providecommand{\remarkname}{Remark}
\providecommand{\theoremname}{Theorem}
\providecommand{\claimname}{Claim}
\providecommand{\remarkname}{Remark}

\begin{document}
\maketitle

\begin{abstract}
We analyze architectural features of Deep Neural Networks (DNNs) using the so-called Neural Tangent Kernel (NTK), which describes the training and generalization of DNNs in the infinite-width setting.
In this setting, we show that for fully-connected DNNs, as the depth grows, two regimes appear: \textit{order}, where the (scaled) NTK converges to a constant, and \textit{chaos}, where it converges to a Kronecker delta. Extreme order slows down training while extreme chaos hinders generalization.
Using the scaled ReLU as a nonlinearity, we end up in the ordered regime.
In contrast, Layer Normalization brings the network into the chaotic regime. We observe a similar effect for Batch Normalization (BN) applied after the last nonlinearity.
We uncover the same order and chaos modes in Deep Deconvolutional Networks (DC-NNs). Our analysis explains the appearance of so-called checkerboard patterns and border artifacts. Moving the network into the chaotic regime prevents checkerboard patterns; we propose a graph-based parametrization which eliminates border artifacts; finally, we introduce a new layer-dependent learning rate to improve the convergence of DC-NNs.
We illustrate our findings on DCGANs: the ordered regime leads to a collapse of the generator to a checkerboard mode, which can be avoided by tuning the nonlinearity to reach the chaotic regime. As a result, we are able to obtain good quality samples for DCGANs without BN.
\end{abstract}

\section{Introduction}

The training of Deep Neural Networks (DNN) involves a great variety
of architecture choices. It is therefore crucial to find tools to
understand their effects and to compare them. For example, Batch Normalization
(BN) \cite{BatchNorm_Ioffe2015} has proven to be crucial in the training
of DNNs but remains ill-understood. While BN was initially introduced
to solve the problem of ``covariate shift'', recent results \cite{HowDoesBN_Santurkar2018}
suggest an effect on the smoothness of the loss surface. Some alternatives
to BN have been proposed \cite{LayerNorm_Ba2016,WeightNorm_Salimans2016,SelfNormalizing_Klambauer2017},
yet it remains difficult to compare them theoretically. Recent theoretical
results \cite{MeanFieldBN_Yang2019} suggest some relation to the
transition from ``order'' to ``chaos'' observed as the depth of
the NN goes to infinity \cite{Chaos_Poole2016,Daniely,Chaos_ResidualNet_Yang2017}.

The impact of architecture is very apparent in GANs \cite{Goodfellow2014}:
their results are heavily affected by the architecture of the generator
and discriminator \cite{DCGAN_radford2015,SAGAN_zhang2018,BIGGAN_brock2018,StyleGAN_Karras2018}
and the training may fail without BN \cite{Norm_prop_Arpit2016,BN_GAN_Xiang2017}.

Recently, there has been important advances \cite{jacot2018neural,Du2019,Allen-Zhu2018}
in the understanding of the training of DNNs when the number of neurons
in each hidden layer is very large. These results give new tools to
study the asymptotic effect of BN. In particular, the Neural Tangent
Kernel (NTK) \cite{jacot2018neural} illustrates the effect of architecture
on the training of DNNs and also describes their loss surface\cite{Karakida2018}.
The NTK can easily be extended to Convolutional Neural Networks (CNNs)
and other architectures \cite{Greg_Yang_2019,exact_arora2019}, hence
allowing comparison. Recently the order/chaos transition has been
extended to the NTK in \cite{hayou2019training,xiao2019disentangling}.

\subsection{Our Contributions}

We describe how the NTK is affected by the ``order'' and ``chaos''
regimes \cite{Chaos_Poole2016,Daniely,Chaos_ResidualNet_Yang2017}.
For fully-connected networks, the scaled NTK converges to a constant
in the ordered regime and to a Kronecker delta in the chaotic regime.
In deconvolutional networks (DC-NNs), a similar transition takes place:
the ordered regime features checkerboard patterns \cite{Checkerboard_odena2016}
and the chaotic regime features a (translation invariant) Kronecker
delta. 

We then show that different normalization techniques such as Layer
Normalization, Batch Normalization and our proposed Nonlinearity Normalization
allows the DNN to avoid the ordered regime.

Besides, we prove that the traditional parametrization of DC-NNs leads
to border effects in the NTK and we propose a simple solution suggesting
a new Graph-Based parametrization. At last, the effect of the number
of channels on the NTK is discussed, giving a theoretical motivation
for decreasing the number of channels after each upsampling. We show
that using a layer-dependent learning rate allows to balance the contributions
of the layers to the learning.

Finally, we demonstrate our findings numerically on DC-GANs: we show
that in the ordered regime, the generator collapses to a checkerboard
mode. We show how a basic DC-GAN can be effectively trained and avoid
this mode collapse: by proper hyperparameter tuning, nonlinearity
normalization, parametrization and learning rate choices, without
using batch normalization, we are able to reach the chaotic regime
and to get good quality samples from a very simple DC-NN generator. 

\section{Fully-Connected Neural Networks}

The first type of architecture we consider are deep Fully-Connected
Neural Networks (FC-NNs). An FC-NN $\mathbb{R}^{n_{0}}\to\mathbb{R}^{n_{L}}$
with nonlinearity $\sigma:\mathbb{R}\to\mathbb{R}$ consists of $L+1$
layers ($L-1$ hidden layers), respectively containing $n_{0},n_{1},\ldots,n_{L}$
neurons. The parameters are the connection weight matrices $W^{\left(\ell\right)}\in\mathbb{R}^{n_{\ell+1}\times n_{\ell}}$
and bias vectors $b^{\left(\ell\right)}\in\mathbb{R}^{n_{\ell+1}}$
for $\ell=0,1,\ldots,L-1$. Following \cite{jacot2018neural}, the
network parameters are aggregated into a single vector $\theta\in\mathbb{R}^{P}$
and initialized using iid standard Gaussians $\mathcal{N}\left(0,1\right)$.
For $\theta\in\mathbb{R}^{P}$, the DNN network function $f_{\theta}:\mathbb{R}^{n_{0}}\to\mathbb{R}^{n_{L}}$
is defined as $f_{\theta}\left(x\right)=\tilde{\alpha}^{\left(L\right)}\left(x\right)$,
where the activations and preactivations $\alpha^{\left(\ell\right)},\tilde{\alpha}^{\left(\ell\right)}$
are recursively constructed using the NTK parametrization: we set
$\alpha^{\left(0\right)}\left(x\right)=x$ and, for $\ell=0,\ldots,L-1$,
\[
\tilde{\alpha}^{\left(\ell+1\right)}\left(x\right)=\frac{\sqrt{1-\beta^{2}}}{\sqrt{n_{\ell}}}W^{\left(\ell\right)}\alpha^{\left(\ell\right)}\left(x\right)+\beta b^{\left(\ell\right)},\qquad\alpha^{\left(\ell+1\right)}\left(x\right)=\sigma\left(\tilde{\alpha}^{\left(\ell+1\right)}\left(x\right)\right),
\]
where $\sigma$ is applied entry-wise and $\beta\geq0$.
\begin{rem}

The hyperparameter $\beta$ allows one to balance the relative contributions
of the connection weights and of the biases during training; in our
numerical experiments, we set $\beta=0.1$. Note that the variance
of the normalized bias $\beta b^{\left(\ell\right)}$ at initialization
can be tuned by $\beta$. 
\end{rem}

\subsection{\label{sec:ntk} Neural Tangent Kernel}

The NTK \cite{jacot2018neural} describes the evolution of $\left(f_{\theta_{t}}\right)_{t\geq0}$
in function space during training. In the FC-NN case, the NTK $\Theta_{\theta}^{\left(L\right)}:\mathbb{R}^{n_{0}}\times\mathbb{R}^{n_{0}}\to\mathbb{R}^{n_{L}\times n_{L}}$
is defined by
\[
\Theta_{\theta,kk'}^{\left(L\right)}\left(x,x'\right)=\sum_{p=1}^{P}\partial_{\theta_{p}}f_{\theta,k}\left(x\right)\partial_{\theta_{p}}f_{\theta,k'}\left(x'\right).
\]
For a dataset $x_{1},\ldots,x_{N}\in\mathbb{R}^{n_{0}}$, we define
the \emph{output} vector $Y_{\theta}=\left(f_{\theta,k}\left(x_{i}\right)\right)_{ik}\in\mathbb{R}^{Nn_{L}}$.
The DNN is trained by optimizing a cost $C:\mathbb{R}^{n_{L}N}\to\mathbb{R}$
through gradient descent, defining a flow $\partial_{t}\theta_{t}=-\nabla_{\theta}C\left(Y_{\theta}\right)\big|_{\theta_{t}}$.
The evolution of the output vector $Y_{\theta}$ can be expressed
in terms of the NTK Gram Matrix $\tilde{\Theta}_{\theta}^{\left(L\right)}=\left(\Theta_{\theta,km}^{\left(L\right)}\left(x_{i},x_{j}\right)\right)_{ik,jm}\in\mathbb{R}^{n_{L}N\times n_{L}N}$
and gradient $\nabla_{Y}C(Y_{\theta_{t}})\in\mathbb{R}^{n_{L}N}$:
\[
\partial_{t}Y_{\theta_{t}}=-\tilde{\Theta}_{\theta}^{(L)}\nabla_{Y}C(Y_{\theta_{t}}).
\]

\subsection{Infinite-Width Limit}

Following \cite{Neal1996,Cho2009,Lee2017}, in the overparametrized
regime at initialization, the preactivations $\left(\tilde{\alpha}_{i}^{\left(\ell\right)}\right)_{i=1,\ldots,n_{\ell}}$
are described by iid centered Gaussian processes with covariance kernels
$\Sigma^{\left(\ell\right)}$ constructed as follows. For a kernel
$K$, set 
\[
\mathbb{L}_{K}^{g}\left(z_{0},z_{1}\right)=\mathbb{E}_{\left(y_{0},y_{1}\right)\sim\mathcal{N}\left(0,\left(K\left(z_{i},z_{j}\right)\right)_{i,j=0,1}\right)}\left[g\left(y_{0}\right)g\left(y_{1}\right)\right].
\]
The \emph{activation kernels} $\Sigma^{\left(\ell\right)}$ are defined
recursively by $\Sigma^{\left(0\right)}\left(z_{0},z_{1}\right)=\beta^{2}+\frac{\left(1-\beta^{2}\right)}{n_{0}}z_{0}^{T}z_{1}$
and $\Sigma^{\left(\ell+1\right)}=\beta^{2}+\left(1-\beta^{2}\right)\mathbb{L}_{\Sigma^{\left(\ell\right)}}^{\sigma}$. 

While random at initialization, in the infinite-width-limit, the NTK
converges to a deterministic limit, which is moreover constant during
training:
\begin{thm}
As $n_{1},\ldots,n_{L-1}\to\infty$, for any $z_{0},z_{1}\in\mathbb{R}^{n_{0}}$
and any $t\geq0$, the kernel $\Theta_{\theta_{t}}^{\left(L\right)}\left(z_{0},z_{1}\right)$
converges to $\Theta_{\infty}^{\left(L\right)}\left(z_{0},z_{1}\right)\otimes\mathrm{Id}_{n_{L}}$,
where $\Theta_{\infty}^{\left(L\right)}\left(z_{0},z_{1}\right)=\sum_{\ell=1}^{L}\Sigma^{\left(\ell\right)}\left(z_{0},z\right)\prod_{l=\ell+1}^{L}\dot{\Sigma}^{\left(l\right)}\left(z_{0},z_{1}\right)$
and $\dot{\Sigma}^{\left(l\right)}=(1-\beta^{2})\mathbb{L}_{\Sigma^{\left(l-1\right)}}^{\dot{\sigma}}$
with $\dot{\sigma}$ denoting the derivative of $\sigma$.
\end{thm}

We refer to \cite{jacot2018neural} for a proof for the sequential
limit $n_{1}\to\infty,\ldots,n_{L-1}\to\infty$ and \cite{Greg_Yang_2019,exact_arora2019}
for the simultaneous limit $\min\left(n_{1},\ldots,n_{L-1}\right)\to\infty$.
As a consequence, in the infinite-width limit, the dynamics of the
labels $Y_{\theta_{t},k}\in\mathbb{R}^{N}$ for each outputs $k$
acquires a simple form in terms of the limiting NTK Gram matrix $\tilde{\Theta}_{\infty}^{(L)}\in\mathbb{R}^{N\times N}$
\[
\partial_{t}Y_{\theta_{t},k}=-\tilde{\Theta}_{\infty}^{(L)}\nabla_{Y_{k}}C(Y_{\theta_{t}}).
\]

\section{\label{sec:freeze-and-chaos}Large Depth Limit}

We now investigate the large $L$ behavior on the NTK (in the infinite-width
limit), revealing a transition between two phases: ``order'' and ``chaos''.
To ensure that the variance of the neurons is constant for all depths
($\Sigma^{(\ell)}(x,x)=1$) we consider \emph{standardized} nonlinearity
$\sigma$ (i.e. such that $\mathbb{E}_{x\sim\mathcal{N}\left(0,1\right)}\left[\sigma^{2}\left(x\right)\right]=1$)
and inputs on the \emph{standard $\sqrt{n}_{0}$-sphere}\footnote{Note that high dimensional datasets tend to concentrate on hyperspheres:
for example in GANs \cite{Goodfellow2014} the inputs of a generator
are vectors of iid $\mathcal{N}(0,1)$ entries which concentrate around
$\mathbb{S}_{n_{0}}$ for large dimensions.} $\mathbb{S}_{n_{0}}=\left\{ x\in\mathbb{R}^{n_{0}}:\|x\|=\sqrt{n_{0}}\right\} $. 

\subsection{Order and Chaos}

For a standardized $\sigma$, the large-depth behavior of the \emph{normalized
NTK} $\vartheta^{\left(L\right)}\left(x,y\right)=\Theta_{\infty}^{\left(L\right)}\left(x,y\right)/\sqrt{\Theta_{\infty}^{\left(L\right)}\left(x,x\right)\Theta_{\infty}^{\left(L\right)}\left(y,y\right)}$
is determined by the \emph{characteristic value
\begin{equation}
r_{\sigma,\beta}=(1-\beta^{2})\mathbb{E}_{x\sim\mathcal{N}\left(0,1\right)}\left[\dot{\sigma}^{2}\left(x\right)\right].\label{eq:def characteristic value r}
\end{equation}
}
\begin{thm}
\label{thm:infinite-depth-fc-nn}Suppose that $\sigma$ is twice differentiable
and standardized.

\textbf{Order: }If $r_{\sigma,\beta}<1$, there exists $C_{1}$ such
that for $x,y\in\mathbb{S}_{n_{0}}$,
\[
1-C_{1}Lr_{\sigma,\beta}^{L}\leq\vartheta^{(L)}\left(x,y\right)\leq1.
\]
\textbf{Chaos: }If $r_{\sigma,\beta}>1$, for $x\neq\pm y$ in $\mathbb{S}_{n_{0}}$,
there exist $h<1$ and $C_{2}$, such that 
\[
\left|\vartheta^{(L)}\left(x,y\right)\right|\leq C_{2}h^{L}.
\]
\end{thm}

Theorem \ref{thm:infinite-depth-fc-nn} shows that in the ordered
regime, the normalized NTK $\vartheta^{(L)}$ converges to a constant
as $L\to\infty$, whereas in the chaotic regime, it converges to a
Kronecker $\delta$ (taking value $1$ on the diagonal, $0$ elsewhere).
This suggests that the training of deep FC-NN is heavily influenced
by the characteristic value: when $r_{\sigma,\beta}<1$, $\Theta^{\left(L\right)}$
becomes constant, thus slowing down the training, whereas when $r_{\sigma,\beta}>1$,
$\Theta^{\left(L\right)}$ is concentrates on the diagonal, ensuring
fast training, but limiting generalization. To train very deep FC-NNs,
it is necessary to lie ``on the edge of chaos'' $r_{\sigma,\beta}=1$
\cite{Chaos_Poole2016,Chaos_ResidualNet_Yang2017}. 

Theorem \ref{thm:infinite-depth-fc-nn} does not apply directly to
the standardized ReLU $\sigma\left(x\right)=\sqrt{2}\max\left(x,0\right)$,
because it is not differentiable in $0$. The characteristic value
for the standardized ReLU is $r_{\sigma,\beta}=1-\beta^{2}$ which
lies in the ordered regime for $\beta>0$ 
\begin{thm}
\label{thm:infinite-depth-fc-nn-relu}With the same notation as in
Theorem \ref{thm:infinite-depth-fc-nn}, taking $\sigma$ to be the
standardized ReLU and $\beta>0$, the NTK is in the ordered regime:
there exists a constant $C$ such that $1-Cr_{\sigma,\beta}^{L/2}\leq\vartheta^{(L)}\left(x,y\right)\leq1$. 
\end{thm}

We observe two interesting (and potentially beneficial) properties
of the standardized ReLU:
\begin{enumerate}
\item Its characteristic value $r_{\sigma,\beta}=1-\beta^{2}$ is very close
to the `edge of chaos' for small $\beta$ and typically with LeCun
initialization the variance of the bias at initialization is $\nicefrac{1}{w}$
for $w$ the width, which roughly corresponds to a choice of $\beta=\nicefrac{1}{\sqrt{w}}$.
\item The rate of convergence to the limiting kernel is smaller ($r_{\sigma,\beta}^{L/2}$)
for the ReLU than for differentiable nonlinearities ($r_{\sigma,\beta}^{L}$)\footnote{Of course the rates of Theorems \ref{thm:infinite-depth-fc-nn} and
\ref{thm:infinite-depth-fc-nn-relu} may not be tight, but from the
proofs in Appendix \ref{subsec:ReLU-FC-NN} one can observe that the
rate of $r_{\sigma,\beta}^{L/2}$ appears as a result of the non-differentiability
of the ReLU.}. 
\end{enumerate}
These observations suggest that an advantage of the ReLU is that the
NTK of ReLU networks converges to its constant limit at a slower rate
and may naturally offer a good tradeoff between generalization and
training speed.

\begin{wrapfigure}{r}{0.5\textwidth}
\centering\vspace{-1.5cm}

\label{fig:NTK_plots}\hspace{-0.5cm}\includegraphics[scale=0.45]{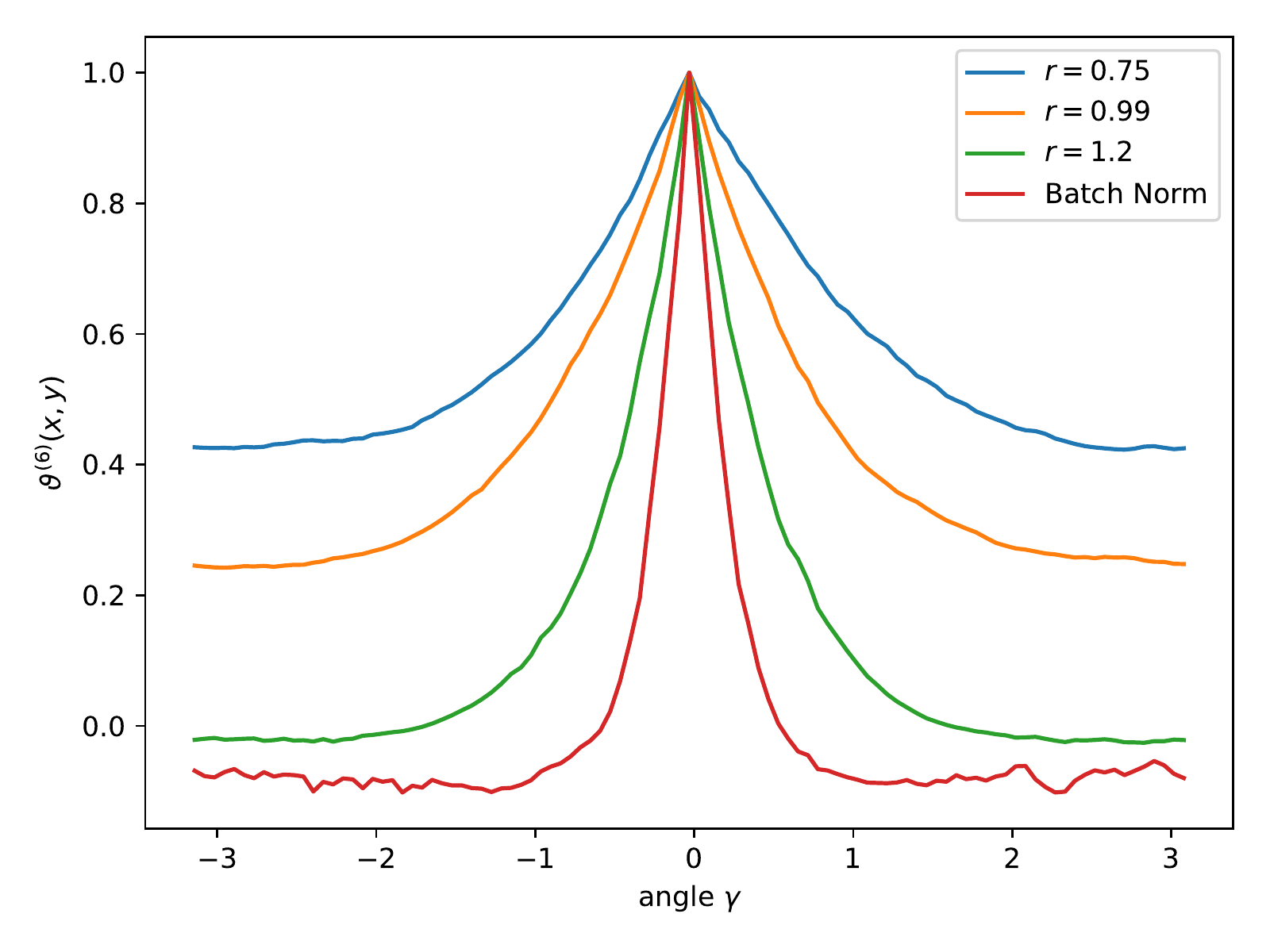}

\caption{The NTK on the unit circle for four architectures with depth $L=6$ are plotted: vanilla ReLU network with $\beta=0.5$ (order) and $\beta=0.1$  (edge of chaos), with a normalized ReLU / Layer norm. (chaos) and with Batch Norm.}
\end{wrapfigure}

\subsection{Chaotic effect of normalization\label{sec:Chaos-Normalization}}

Figure \ref{fig:NTK_plots} shows that even on the edge of chaos,
the NTK may exhibit a strong constant component (i.e. $\vartheta(x,y)>0.2$
for all $x,y$) which can lead to a bad conditioning of the Gram matrix
governing the infinite-width training behavior. It may be helpful
to slightly 'move' the network towards the chaotic regime to reduce
this effect. In Figure \ref{fig:NTK_plots}, $r_{\sigma,\beta}$ plays
a similar role to that of the lengthscale parameter in classical kernel
methods: increasing $r_{\sigma,\beta}$ makes the NTK 'narrower',
reducing the correlation length.

From the definition \ref{eq:def characteristic value r} of the characteristic
value, we see that increasing the bias pushes the network towards
the ordered regime, whereas $r_{\sigma,\beta}$ reaches its highest
value $\mathbb{E}\left[\dot{\sigma}^{2}\left(x\right)\right]$ when
the bias is $0$, which may still be in the ordered regime (or on
the edge with the ReLU). We are therefore interested in ways to push
the network further towards the chaotic regime.

In this section, we show that Layer Normalization is asymptotically
equivalent to Nonlinearity Normalization which entails $r_{\sigma,\beta}>1$
for $\beta$ small enough. While Batch normalization cannot be directly
interpreted in terms of $r_{\sigma,\beta}$, it is easy to show that
it directly controls the constant component of the NTK, which is characteristic
of the ordered regime.

\subsubsection{Nonlinearity Normalization and Layer Normalization}

Intuitively, the dominating constant component in ReLU networks is
partly a consequence of the ReLU being non-negative: after the first
hidden layer, all negative correlations become positive (i.e. $\Sigma^{(1)}(x,y)\geq\beta$
for all $x,y$, even $x=-y$). One can address this issue thanks to
the following. We shall write $Z$ for a random variable with standard
normal distribution. We say that $\sigma$ is normalized if $\mathbb{{E}}[\sigma(Z)]=0$
and $\mathbb{{E}}[\sigma(Z)^{2}]=1$. In particular, if $\sigma\neq\mathrm{{id}}$,
then $\overline{\sigma}(\cdot):=\frac{{\sigma(\cdot)-\mathbb{{E}}[\sigma(Z)]}}{\sqrt{{\mathbb{{E}}[(\sigma(Z)-\mathbb{{E}}[\sigma(Z)])^{2}]}}}$
is normalized. By Poincar� Inequality, after nonlinearity normalization,
one can always reach the chaotic regime:
\begin{prop}
If $\sigma\neq\mathrm{id}$ is normalized, then $\mathbb{E}\left[\dot{\sigma}^{2}\left(Z\right)\right]>1$
and $r_{\sigma,\beta}>1$ for $\beta>0$ small enough.
\end{prop}

With post-nonlinearity layer normalization, for every input $x\in\mathbb{R}^{n_{0}}$
and every $\ell=1,\ldots,L-1$, the activations become
\[
\check{\alpha}^{\left(\ell\right)}\left(x\right)=\mathrm{LN}(\alpha^{(\ell)}(x))=\sqrt{n_{\ell}}\frac{\alpha^{\left(\ell\right)}\left(x\right)-\underline{\mu}^{(\ell)}(x)}{\|\alpha^{\left(\ell\right)}\left(x\right)-\underline{\mu}^{(\ell)}(x)\|},
\]
where $\underline{\mu}^{(\ell)}(x)=\frac{1}{n_{\ell}}\sum_{i=1}^{n_{\ell}}\alpha_{i}^{\left(\ell\right)}\left(x\right)\begin{pmatrix}1\ \cdots\ 1\end{pmatrix}^{\mathrm{T}}$.
We define similarly pre-nonlinearity layer normalization with $\check{\alpha}^{\left(\ell\right)}\left(x\right)=\mathrm{\sigma(LN}(\widetilde{\alpha}^{(\ell)}(x)))$.
\begin{prop}
\label{prop:LN equiv NN}Suppose that the inputs belong to $\mathbb{S}_{n_{0}}$.
Asymptotically, as the widths of the network sequentially go to infinity,
post-nonlinearity layer normalization is equivalent to nonlinearity
normalization, whereas pre-nonlinearity layer normalization has no
effect.
\end{prop}

\subsubsection{Batch Normalization}

For any $N\times d$ matrix of features $X$ leading to a $N\times N$
Gram matrix $K=\frac{1}{d}XX^{T}$, the Rayleigh quotient $\frac{1}{N}\mathbf{1}^{T}K\mathbf{1}$
of the constant vector $\mathbf{1}$ measures how big the constant
component is. Applying Batch Normalization (BN) at a layer $\ell$
centers (and standardizes) the activations\footnote{We consider here \emph{post-nonlinearity }BN, it is common to normalize
the pre-activations $\tilde{\alpha}^{(\ell)}$ instead.} $\alpha_{j}^{(\ell)}(x_{i})$ over a batch $x_{1},...,x_{N}$ , thus
zeroing the constant Rayleigh quotient of the $N\times N$ features
Gram matrices $\tilde{\Sigma}^{(\ell)}$ with entries $\tilde{\Sigma}_{ij}^{(\ell)}=\frac{1}{n_{\ell}}\sum_{k=1}^{n_{\ell}}\alpha_{k}^{(\ell)}(x_{i})\alpha_{k}^{(\ell)}(x_{j})$.
Adding a single BN layer after the last hidden layer controls the
constant Rayleigh quotient of the NTK Gram matrix $\tilde{\Theta}^{(L)}$:
\begin{lem}
Consider FC-NN with $L$ layers, with a PN-BN after the last nonlinearity
then $\frac{1}{N}\mathbf{1}^{T}\tilde{\Theta}^{(L)}\mathbf{1}=\beta^{2}$.
\end{lem}

In contrast, for a network in the extreme order, i.e. such that $\Theta^{(L)}(x,y)\approx c$
for some constant $c>0$, the constant Rayleigh quotient scales as
$\frac{1}{N}\mathbf{1}^{T}\tilde{\Theta}^{(L)}\mathbf{1}\approx cN$.
The analysis of BN presented in \cite{karakida2019normalization}
is also closely related to this phenomenon.

\section{Convolutional Networks and Generative Adversarial Networks\label{sec:DC-NN}}

The order/chaos transition is even more interesting for convolutional
networks, in particular in the context of Generative Adversarial Networks
(GANs): a common problem in GAN training is the so-called `mode collapse',
where the generator converges to a constant function, hence generating
a single image instead of a variety of images. This problem is closely
related to the fact that the constant mode of the NTK Gram matrix
dominates, and indeed the problem of mode collapse is most proeminent
in the ordered regime (Figure \ref{fig:NTK_PCA}), while normalization
techniques (leading to a chaotic network) mitigate this problem. 

In this section, we use the NTK to explain the appearance of border
and checkerboard artifacts in generated images. We show that the border
artifacts issue can be solved by a change of parametrization and that
the checkerboard artifacts occur in the ordered regime, and can hence
be avoided by adding normalization and using layer-wise learning rates.
With these changes we are able to train GANs on CelebA dataset without
Batch Normalization.

\subsection{Graph-based Neural Networks (GB-NNs)\label{subsec:Graph-based-Neural-Networks-setup}}

In FC-NNs, the neurons are indexed by their layer $\ell$ and their
channel $i\in\left\{ 1,...,n_{\ell}\right\} $, in convolutional networks
each neuron furthermore has a location on the image (or on a downscaled
image). The position $p$ of a neuron determines its connections with
the neurons of the previous and subsequent layers. Furthermore certain
connections are shared, i.e. they evolve together. We abstract these
concepts in the following manner: 

For each layer $\ell=0,...,L$, the neurons are indexed by a position
$p\in I_{\ell}$ and a channel $i=1,...,n_{\ell}$. The sets of positions
$I_{\ell}$ can be any set, in particular any subset of $\mathbb{Z}^{D}$.
Each position $p\in I_{\ell+1}$ has a set of parents $P(p)\subset I_{\ell}$
which are neurons of the previous layer connected to $p$. The connections
from the parent $\left(q,\ell\right)$ to the position $\left(p,\ell+1\right)$
are encoded in an $n_{\ell}\times n_{\ell+1}$ weight matrix $W^{(\ell,q\to p)}$.
Finally two connections $q\to p$ and $q'\to p'$ can be shared, setting
the corresponding matrices to be equal $W^{(\ell,q\to p)}=W^{(\ell,q'\to p')}$.

The inputs of the network $x$ are vectors in $\left(\mathbb{R}^{n_{0}}\right)^{I_{0}}$,
for example for colour images of width $w$ and height $h$, we have
$n_{0}=3$ and $I_{0}=\left\{ 1,...,w\right\} \times\left\{ 1,...,h\right\} \subset\mathbb{Z}^{2}$.
The activations and preactivations $\alpha^{\left(\ell\right)},\tilde{\alpha}^{\left(\ell\right)}\in\left(\mathbb{R}^{n_{\ell}}\right)^{I_{\ell}}$
are constructed recursively using the NTK parametrization: we set
$\alpha^{\left(0,p\right)}\left(x\right)=x^{(p)}$ and for $\ell=0,\ldots,L-1$
and any position $p\in I_{\ell+1}$,
\begin{equation}
\tilde{\alpha}^{(\ell+1,p)}(x)=\beta b^{(\ell)}+\frac{\sqrt{1-\beta^{2}}}{\sqrt{\left|P(p)\right|n_{\ell}}}\sum_{q\in P(p)}W^{(\ell,q\to p)}\alpha^{(\ell,q)}(x),\qquad\alpha^{\left(\ell+1,p\right)}\left(x\right)=\sigma\left(\tilde{\alpha}^{\left(\ell+1,p\right)}\left(x\right)\right)\label{eq:def-graph-based}
\end{equation}
where $\sigma$ is applied entry-wise, $\beta\geq0$ and $\left|P(p)\right|$
is the cardinality of $P(p)$.

\subsubsection{Deconvolutional networks}

Deconvolutional networks (DC-NNs) in dimension $D$ can be seen as
a special case of GB-NNs. We first consider borderless DC-NNs, i.e.
the set of positions are $I_{\ell}=\mathbb{Z}^{D}$ for all layers
$\ell$. Given window dimensions $(w_{1},...,w_{D})$ and strides
$(s_{1},...,s_{D})$, the set of parents of $p\in I_{\ell+1}$ is
the hyperrectangle $P(p)=\left\{ \left\lfloor p_{1}/s_{1}\right\rfloor +1,...,\left\lfloor p_{1}/s_{1}\right\rfloor +w_{1}\right\} \times\cdots\times\left\{ \left\lfloor p_{D}/s_{D}\right\rfloor +1,...,\left\lfloor p_{D}/s_{D}\right\rfloor +w_{D}\right\} \subset\mathbb{Z}^{D}$.
Two connections $q\to p$ and $q'\to p'$ are shared if $s_{d}\mid p_{d}-p_{d}'$
and $q_{d}-q_{d}'=\frac{p_{d}-p_{d}'}{s_{d}}$ for all $d=1,...,D$.
This definition can easily be extended to any other choices of position
sets $I_{\ell}\subset\mathbb{Z}^{D}$ (for example hyperrectangles)
by considering $P(p)\cap I_{\ell}$ in place of $P(p)$ as parents
of $p$.

\subsubsection{Neural Tangent Kernel}

As for FC-NNs , in the infinite width limit (when $n_{1},...,n_{L-1}\to\infty$)
the preactivations $\tilde{\alpha}_{i}^{(\ell,p)}(x)$ converge to
Gaussian processes with covariance 
\[
Cov\left(\tilde{\alpha}_{i}^{(\ell+1,p)}(x),\tilde{\alpha}_{j}^{(\ell+1,q)}(y)\right)=\delta_{ij}\Sigma^{(\ell,pq)}(x,y).
\]
The behavior of the network during training is described by the NTK
\[
\Theta_{ij}^{(\ell,pq)}(x,y)=\sum_{k=1}^{P}\partial_{\theta_{k}}\tilde{\alpha}_{i}^{(\ell+1,p)}(x)\partial_{\theta_{k}}\tilde{\alpha}_{j}^{(\ell+1,q)}(y).
\]
In Section \ref{sec:General-Convolutional-Networks} of the Appendix
we prove the convergence $\Theta_{ij}^{(\ell,pq)}(x,y)\to\delta_{ij}\Theta_{\infty}^{(\ell,pq)}(x,y)$
of the NTK and give formulas for the limiting kernels $\Sigma^{(\ell,pq)}(x,y)$
and $\Theta_{\infty}^{(\ell,pq)}(x,y)$. 

\subsubsection{Border Effects}

A very important element of the NTK parametrization proposed in Section
\ref{subsec:Graph-based-Neural-Networks-setup} is the factors $1/\sqrt{\left|P(p)\right|n_{\ell}}$
in the definition of the preactivation (Equation \ref{eq:def-graph-based}):
we scale the contribution of the previous layer according to the number
of neurons $\left|P(p)\right|n_{\ell}$ (i.e. $n_{\ell}$ channels
for each of the$\left|P(p)\right|$ positions) which are fed into
the neuron. For inputs $x\in\mathbb{S}_{n_{0}}^{I_{0}}$ (i.e. such
that $x^{(p)}\in\mathbb{S}_{n_{0}}$ for all $p$), these factors
ensure that the limiting variance $\Sigma^{\left(\ell,pp\right)}\left(x,x\right)$
of $\tilde{\alpha}_{i}^{(\ell,p)}(x)$ at initialization is the same
for all $p$:
\begin{prop}
\label{prop: GB-NNs NTK variance}For GB-NNs with the NTK parametrization,
$\Sigma^{\left(\ell,pp\right)}\left(x,x\right)$ and $\Theta_{\infty}^{\left(\ell,pp\right)}\left(x,x\right)$
do not depend neither on $p\in I_{\ell}$ nor on $x\in\mathbb{S}_{n_{0}}^{I_{0}}$.
\end{prop}

These factors are usually not present and to compensate, the variance
of the weights at initialization is reduced. In convolutional networks
with LeCun initialization, the standard deviation of the weights at
initialization is set to $\frac{1}{\sqrt{whn_{\ell}}}$ for $w$ and
$h$ the width and height of the window of convolution, which has
roughly the effect of replacing the $\frac{1}{\sqrt{\left|P(p)\right|n_{\ell}}}$
factors by $\frac{1}{\sqrt{whn_{\ell}}}$. However $whn_{\ell}$ is
the maximal number of parents that a neuron can have, it is typically
attained at positions $p$ in the middle of the image. Positions $p$
on the border of the image have less parents hence leading to a smaller
contribution of the previous layer. This leads both kernels $\Sigma^{(\ell,pp)}(x,x)$
and $\Theta^{(\ell,pp)}(x,x)$ to have lower intensity for $p\in I_{\ell}$
on the border (see Appendix G for an example when $I_{\ell}=\mathbb{N}$,
i.e. when there is one border pixel), leading to border artifacts
as seen in Figure \ref{fig:NTK_PCA}.

\begin{figure}
\centering \rotatebox{90}{\;\;\;\;\;\;\;\;\;\;\; ORDER}\includegraphics[scale=0.5]{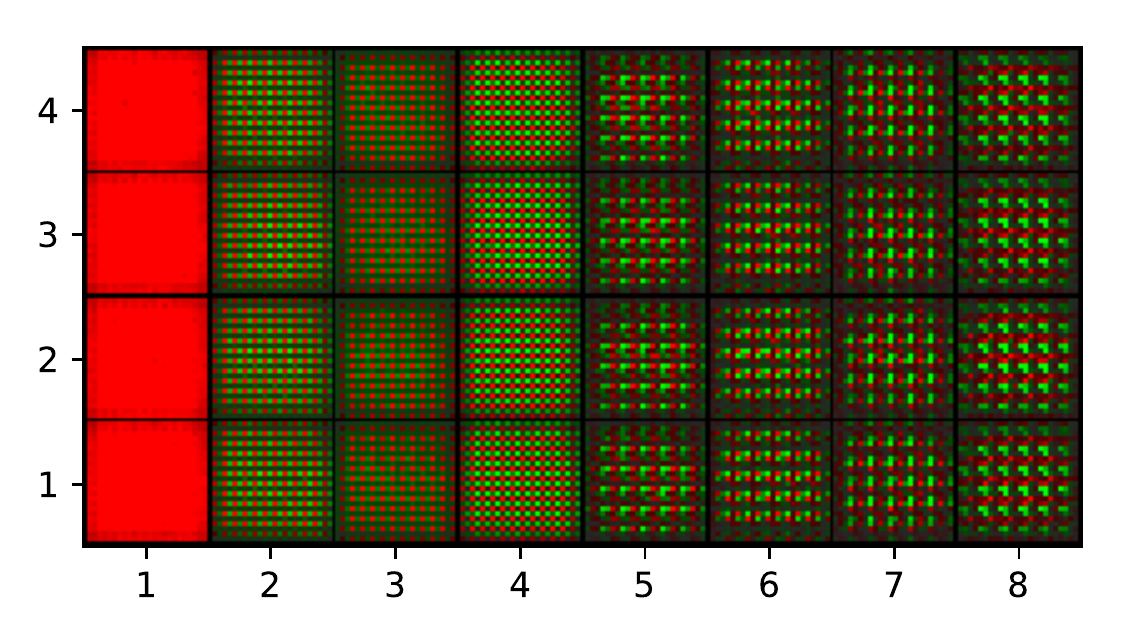}\includegraphics[scale=0.5]{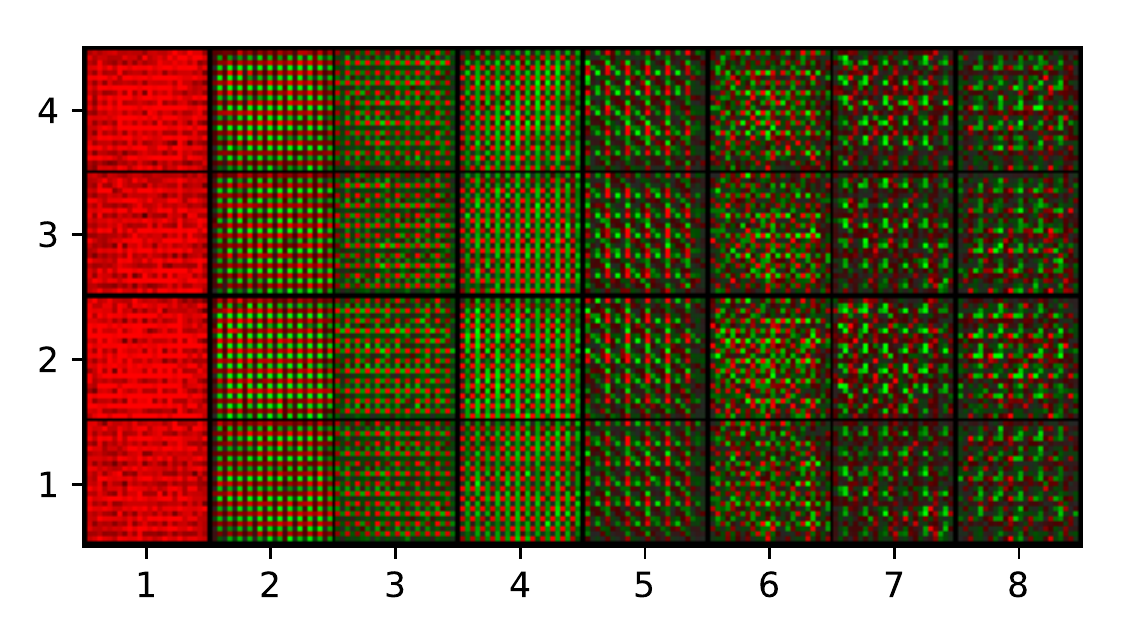}\includegraphics[viewport=0bp -65bp 326bp 390bp,scale=0.185]{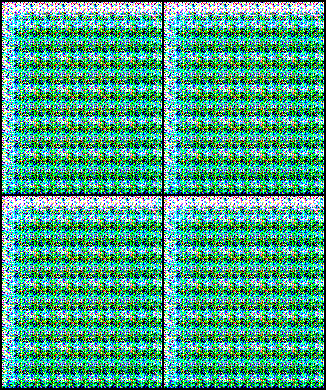}

\rotatebox{90}{\;\;\;\;\;\;\;\;\;\;\; CHAOS}\includegraphics[scale=0.5]{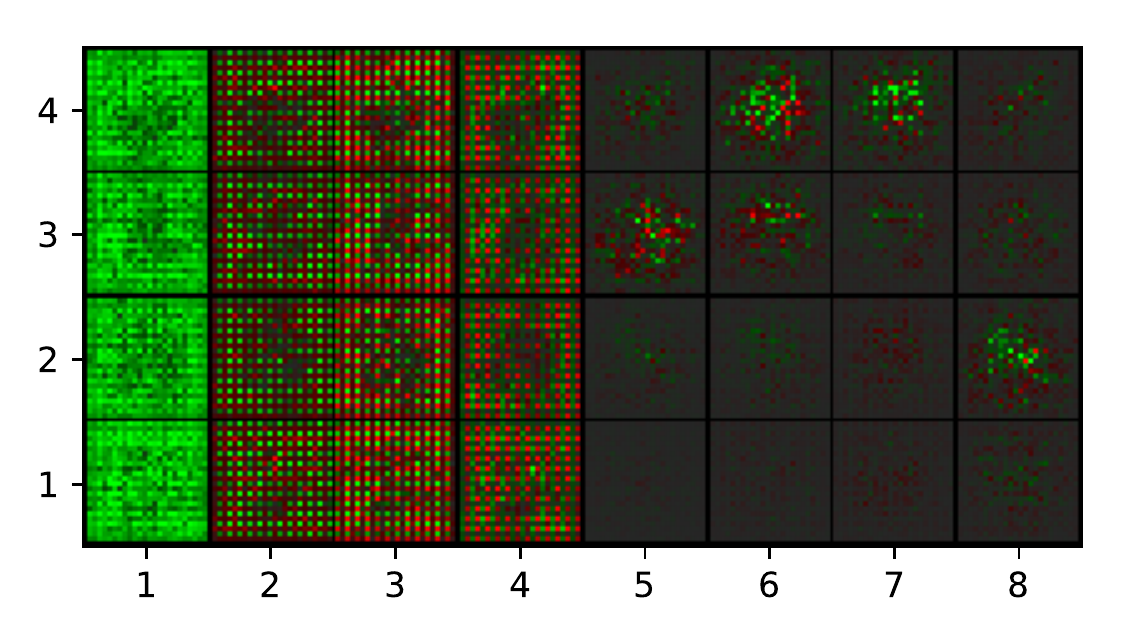}\includegraphics[scale=0.5]{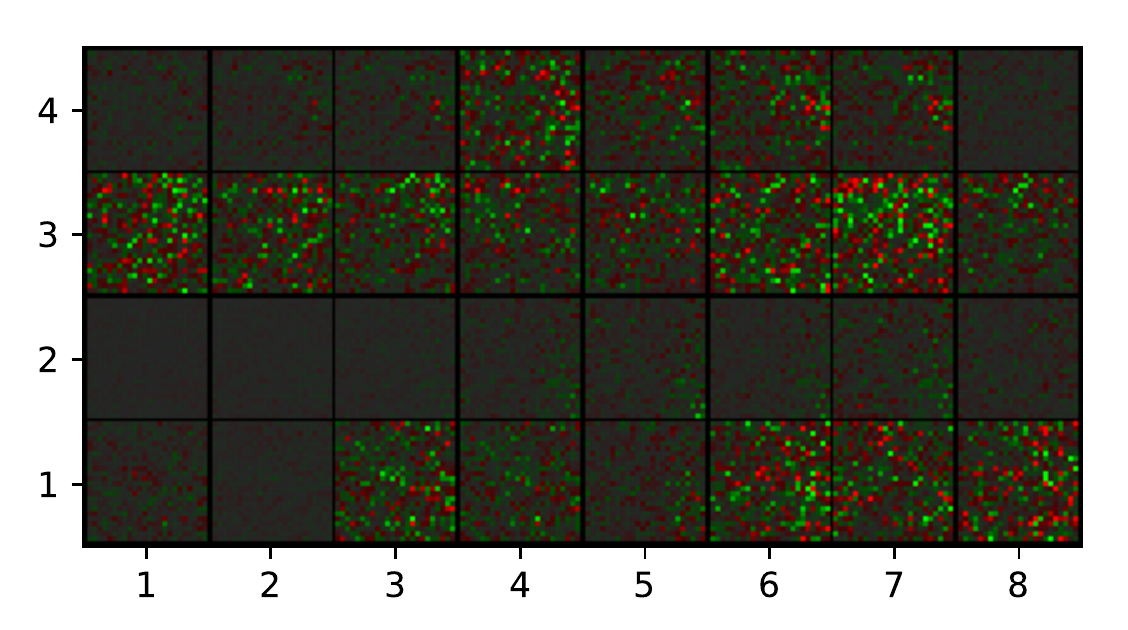}\includegraphics[viewport=0bp -65bp 326bp 390bp,scale=0.185]{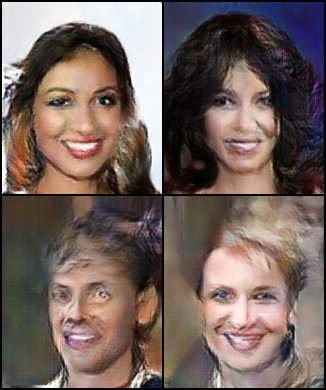}\\
\rotatebox{90}{\;\;\; BATCH NORM}\includegraphics[scale=0.5]{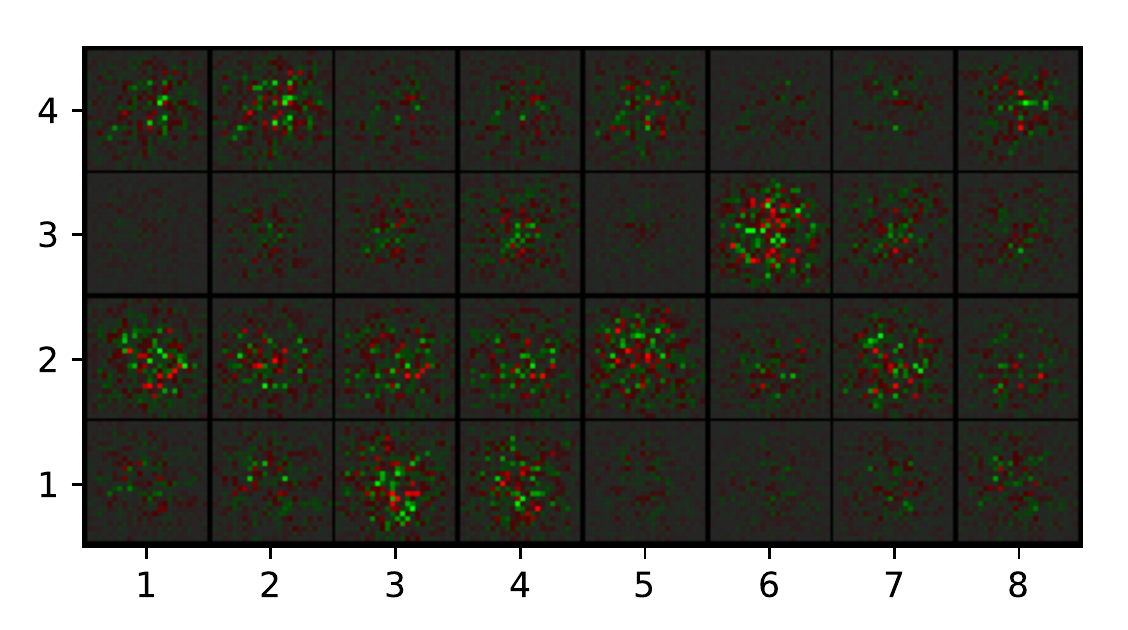}\includegraphics[scale=0.5]{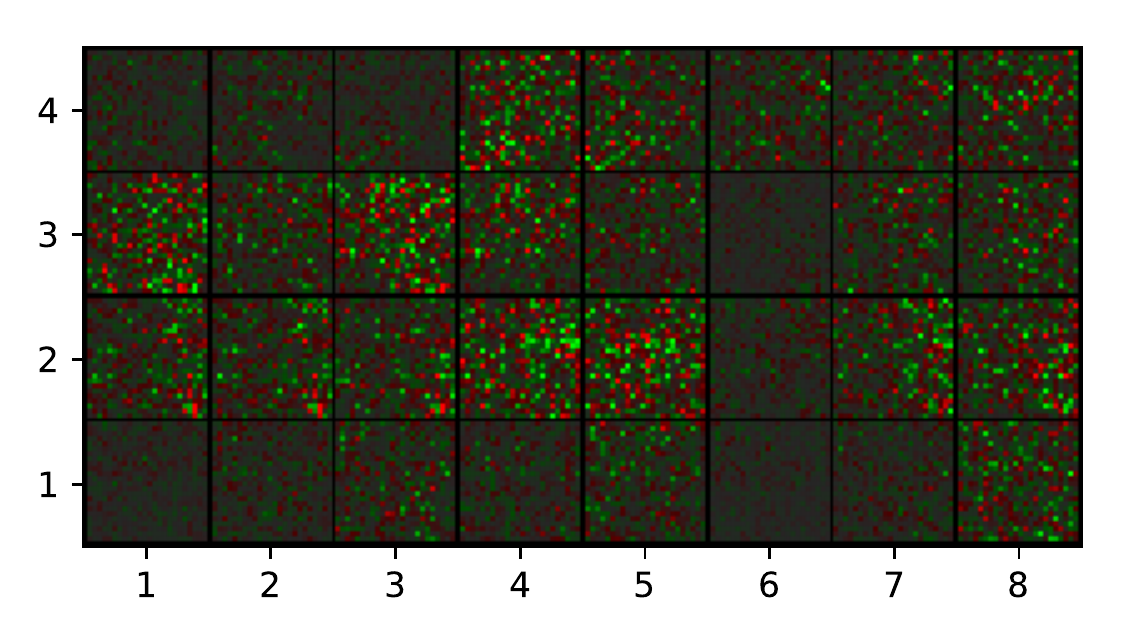}\includegraphics[viewport=0bp -65bp 326bp 390bp,scale=0.185]{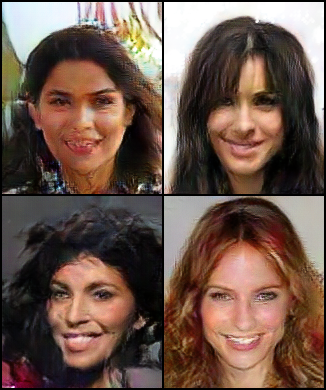}

\;\;\;\;\;\;\;\;\;\;\;\;\;\;\;\;\;\; standard \;\;\;\;\;\;\;\;\;\;\;\;\;\;\;\;\;\;\;\;\;\;\;\; graph-based + layer dependent lr. \;\;\;\;\;\;\;  GAN 

\caption{\label{fig:NTK_PCA}The left and middle columns represent the first
8 eigenvectors of the NTK Gram matrix of a DC-NN (L=3) on 4 inputs.
(left) without the Graph-Based Parametrization (GBP) and the Layer-Dependent
Learning Rate (LDLR); (middle) with GBP and LDLR. The right column
represents the results of a GAN on CelebA with GBP and LDLR. Each
line correspond to a choice of nonlinearity/normalization for the
generator: (top) ReLU, (middle) normalized ReLU and (bottom) ReLU
with Batch Normalization.}
\end{figure}

\subsection{Bulk Order and Chaos for Deconvolutional Nets}

Large depths deconvolutional networks exhibit a similar Order/Chaos
transition as that of FC-NNs, the values of the limiting kernel at
different positions $\Theta^{(L,pq)}$ is especially interesting.

For GB-NNs, the value of an output neuron at a position $p\in I_{L}$
only depends on the inputs which are ancestors of $p$, i.e. all positions
$q\in I_{0}$ such that there is a chain of connections from $q$
to $p$. For the same reason , the NTK $\Theta^{(L,pp')}(x,y)$ only
depends on the values $x_{q},y_{q'}$ for $q,q'\in I_{0}$ ancestors
of $p$ and $p'$ respectively.

For a stride $s\in\left\{ 2,3,\ldots\right\} ^{d}$, we denote the
$s$-valuation $v_{s}\left(n\right)$ of $n\in\mathbb{Z}^{d}$ as
the largest $k\in\left\{ 0,1,2,\ldots\right\} $ such that $s_{i}^{k}\mid n_{i}$
for all $i=1,...,d$. The behaviour of the NTK $\Theta_{p,p'}^{(L)}(x,y)$
depends on the $s$-valuation of the difference of the two output
positions. If $v_{s}\left(p'-p\right)$ is strictly smaller than $L$,
the NTK $\Theta^{(L,pp')}(x,y)$ converges to a constant in the infinite-width
limit for any $x,y\in\mathbb{S}_{n_{0}}^{I_{0}}$. Again the characteristic
value $r_{\sigma,\beta}$ plays a central role in the behavior of
the large-depth limit. In this context, we define the rescaled NTK
as $\vartheta^{\left(L,pp'\right)}\left(x,y\right)=\Theta^{(L,pp')}(x,y)/\sqrt{{\Theta^{(L,pp)}(x,x)\Theta^{(L,p'p')}(y,y)}}$
(note that the denominator actually does not depend on $p,p',x$ nor
$y$ by Proposition \ref{prop: GB-NNs NTK variance})
\begin{thm}
\label{thm:freeze_chaos_DC-NN}Consider a borderless DC-NN with position
sets $I_{\ell}=\mathbb{Z}^{D}$ for all layers $\ell$, upsampling
stride $s\in\left\{ 2,3,\ldots\right\} ^{D}$ and window sizes $w\in\left\{ 1,2,3,\ldots\right\} ^{D}$.
For a standardized twice differentiable $\sigma$, there exist constants
$C_{1},C_{2}>0,$ such that the following holds: for $x,y\in\mathbb{S}_{n_{0}}^{I_{0}}$,
and any positions $p,p'\in I_{L}$, we have

\textbf{Order:} When $r_{\sigma,\beta}<1$, taking $v=\min\left(v_{s}\left(p-p'\right),L-1\right)$%
, we have 
\[
\frac{1-r_{\sigma,\beta}^{v+1}}{1-r_{\sigma,\beta}^{L}}-C_{1}(v+1)r_{\sigma,\beta}^{v}\leq\vartheta^{\left(L,pp'\right)}\left(x,y\right)\leq\frac{1-r_{\sigma,\beta}^{v+1}}{1-r_{\sigma,\beta}^{L}}.
\]
\textbf{Chaos:} When $r_{\sigma,\beta}>1$, if either $v_{s}\left(p-p'\right)<L$
or if there exists $c<1$ such that for all positions $q\in I_{0}$
which are ancestors of $p$, $\left|x_{q}^{T}y_{q+\frac{p'-p}{s^{L}}}\right|<c$,
then there exists $h<1$ such that
\[
\left|\vartheta^{\left(L,pp'\right)}\left(x,y\right)\right|\leq C_{2}h^{L}.
\]
\end{thm}

This theorem suggests that in the order regime, the correlations between
differing positions $p$ and $p'$ increase with $v_{s}\left(p-p'\right)$,
which is a strong feature of checkerboard patterns \cite{Checkerboard_odena2016}.
These artifacts typically appear in images generated by DC-NNs. The
form of the NTK also suggests a strong affinity to these checkerboard
patterns: they should dominate the NTK spectral decomposition. This
is shown in Figure \ref{fig:NTK_PCA} where the eigenvectors of the
NTK Gram matrix for a DC-NN are computed.

In the chaotic regime, the normalized NTK converges to a ``scaled
translation invariant'' Kronecker delta. For two output positions
$p$ and $p'=p+ks^{L}$ we associate the two regions $\omega$ and
$\omega'=\omega+k$ of the input space which are connected to $p$
and $p'$. Then $\vartheta^{\left(L,p,p+ks^{L}\right)}\left(x,y\right)$
is one if the patch $y_{\omega'}$ is a $k$ translation of $x_{\omega}$
and approximately zero otherwise.

\subsubsection{Layer-dependent learning rate\label{subsec:Layer-dependent-learning-rate}}

The NTK is the sum $\Theta^{(L)}=\sum_{\ell}\Theta_{W^{(\ell)}}^{(L)}+\Theta_{b^{(\ell)}}^{(L)}$
over the contributions of the weights $\Theta_{W^{(\ell)}}^{(L,pq)}(x,y)=\sum_{ij}\partial_{W_{ij}^{(\ell)}}f_{\theta,p}(x)\partial_{W_{ij}^{(\ell)}}f_{\theta,q}(y)$
and biases $\Theta_{b^{(\ell)}}^{(L,pq)}(x,y)=\sum_{j}\partial_{b_{j}^{(\ell)}}f_{\theta,p}(x)\partial_{b_{j}^{(\ell)}}f_{\theta,q}(y)$.
At the $\ell$-th layer, the weights and biases can only contribute
to checkerboard patterns of degree $v=L-\ell$ and $v=L-\ell-1$,
i.e. patterns with periods $s^{L-\ell}$ and $s^{L-\ell-1}$ respectively,
in the following sense:
\begin{prop}
In a DC-NN with stride $s\in\{2,3,...\}^{d}$, we have $\Theta_{\infty,W^{(\ell)}}^{(L,pp')}(x,y)=0$
if $s^{L-\ell}\nmid p'-p$ and $\Theta_{\infty,b^{(\ell)}}^{(L,pp')}(x,y)=0$
if $s^{L-\ell-1}\nmid p'-p$.
\end{prop}

This suggests that the supports of $\Theta_{\infty,W^{(\ell)}}^{(L)}$
and $\Theta_{\infty,b^{(\ell)}}^{(L)}$ increase exponentially with
$\ell$, giving more importance to the last layers during training.
This could explain why the checkerboard patterns of lower degree dominate
in Figure \ref{fig:NTK_PCA}. In the classical parametrization, the
balance is restored by letting the number of channels $n_{\ell}$
decrease with depth \cite{DCGAN_radford2015}. In the NTK parametrization,
the limiting NTK is not affected by the ratios $\nicefrac{n_{\ell}}{n_{k}}$.
To achieve the same effect, we divide the learning rate of the weights
and bias of the $\ell$-th layer by $S^{\nicefrac{\ell}{2}}$ and
$S^{\nicefrac{(\ell+1)}{2}}$ respectively, where $S=\prod_{i}s_{i}$
is the product of the strides. Together with the 'parent-based' parametrization
and the normalization of the nonlinearity (in order to lie in the
chaotic regime) this rescaling of the learning rate removes both border
and checkerboard artifacts in Figure \ref{fig:NTK_PCA}.

\section{Conclusion}

This article shows how the NTK can be used theoretically to understand
the effect of architecture choices (such as decreasing the number
of channels or batch normalization) on the training of DNNs. We have
shown that DNNs in a ``order'' regime, have a strong affinity to
constant modes and checkerboard artifacts: this slows down training
and can contribute to a mode collapse of the DC-NN generator of GANs.
We introduce simple modifications to solve these problems: the effectiveness
of normalizing the nonlinearity, a parent-based parametrization and
a layer-dependent learning rates is shown both theoretically and numerically.

\section*{Broader Impact}

This work is theoretical and has as such no direct social impact.

\bibliographystyle{plain}
\bibliography{main}

\appendix

\section{Choice of Parametrization}

The NTK parametrization introduced in Section 2 differs slightly from
the one usually used, yet it ensures that the training is consistent
as the size of the layers grows. In the standard parametrization,
the activations are defined by 
\begin{align*}
\alpha^{(0)}(x;\theta) & =x\\
\tilde{\alpha}^{(\ell+1)}(x;\theta) & =W^{(\ell)}\alpha^{(\ell)}(x;\theta)+b^{(\ell)}\\
\alpha^{(\ell+1)}(x;\theta) & =\sigma\left(\tilde{\alpha}^{(\ell+1)}(x;\theta)\right).
\end{align*}
Let denote by $g_{\theta}$ the output function of the DNN thus parametrized,
and $f_{\theta}$ that of the DNN with NTK parametrization. Note the
absence of $\nicefrac{1}{\sqrt{n_{\ell}}}$ in comparison to the NTK
parametrization. With LeCun/He initialization \cite{init_lecun2012},
the parameters $W^{(\ell)}$ have standard deviation $\nicefrac{1}{\sqrt{n_{\ell}}}$
(or $\nicefrac{\sqrt{2}}{\sqrt{n_{\ell}}}$ for the ReLU but this
does not change the general analysis). Using this initialization,
the activations stay stochastically bounded as the widths of the DNN
get large. In the forward pass, there is almost no difference between
the two parametrizations and for each choice of parameters $\theta$,
we can scale down the connection weights by $\nicefrac{\sqrt{1-\beta^{2}}}{\sqrt{n_{\ell}}}$
and the bias weights by $\beta$ to obtain a new set of parameters
$\hat{\theta}$ such that 
\[
f_{\theta}=g_{\hat{\theta}}.
\]

The two parametrizations will exhibit a difference during backpropagation
since:
\[
\partial_{W_{ij}^{(\ell)}}g_{\hat{\theta}}(x)=\frac{\sqrt{n_{\ell}}}{\sqrt{1-\beta^{2}}}\partial_{W_{ij}^{(\ell)}}f_{\theta}(x),\qquad\partial_{b_{j}^{(\ell)}}g_{\hat{\theta}}(x)=\frac{1}{\beta}\partial_{b_{j}^{(\ell)}}f_{\theta}(x).
\]
The NTK is a sum of products of these derivatives over all parameters:
\[
\Theta^{(L)}=\Theta^{(L:W^{(0)})}+\Theta^{(L:b^{(0)})}+\Theta^{(L:W^{(1)})}+\Theta^{(L:b^{(1)})}+...+\Theta^{(L:W^{(L-1)})}+\Theta^{(L:b^{(L-1)})}.
\]
With our parametrization, all summands converge to a finite limit,
while with the Le Cun or He parameterization we obtain 
\[
\hat{\Theta}^{(L)}=\frac{n_{0}}{1-\beta^{2}}\Theta^{(L:W^{(0)})}+\frac{1}{\beta^{2}}\Theta^{(L:b^{(0)})}+...+\frac{n_{L-1}}{1-\beta^{2}}\Theta^{(L:W^{(L-1)})}+\frac{1}{\beta^{2}}\Theta^{(L:b^{(L-1)})},
\]
where some summands, namely the $\left(\frac{n_{i}}{1-\beta^{2}}\Theta^{(L:W^{(i)})}\right)_{i},$
explode in the infinite width limit. One must therefore take a learning
rate of order $\nicefrac{1}{\max(n_{1},...n_{L-1})}$ \cite{Karakida2018,Park2018}
to obtain a meaningful training dynamics, but in this case the contributions
to the NTK of the first layers connections $W^{(0)}$ and the bias
of all layers $b^{(\ell)}$ vanish, which implies that training these
parameters has less and less effect on the function as the width of
the network grows. As a result, the dynamics of the output function
during training can still be described by a modified kernel gradient
descent: the modified learning rate compensates for the absence of
normalization in the usual parametrization.

The NTK parametrization is hence more natural for large networks,
as it solves both the problem of having meaningful forward and backward
passes, and to avoid tuning the learning rate, which is the problem
that sparked multiple alternative initialization strategies in deep
learning \cite{Glorot2010}. Note that in the standard parametrization,
the importance of the bias parameters shrinks as the width gets large;
this can be implemented in the NTK parametrization by taking a small
value for the parameter $\beta$.

\begin{figure}
\includegraphics[scale=0.18]{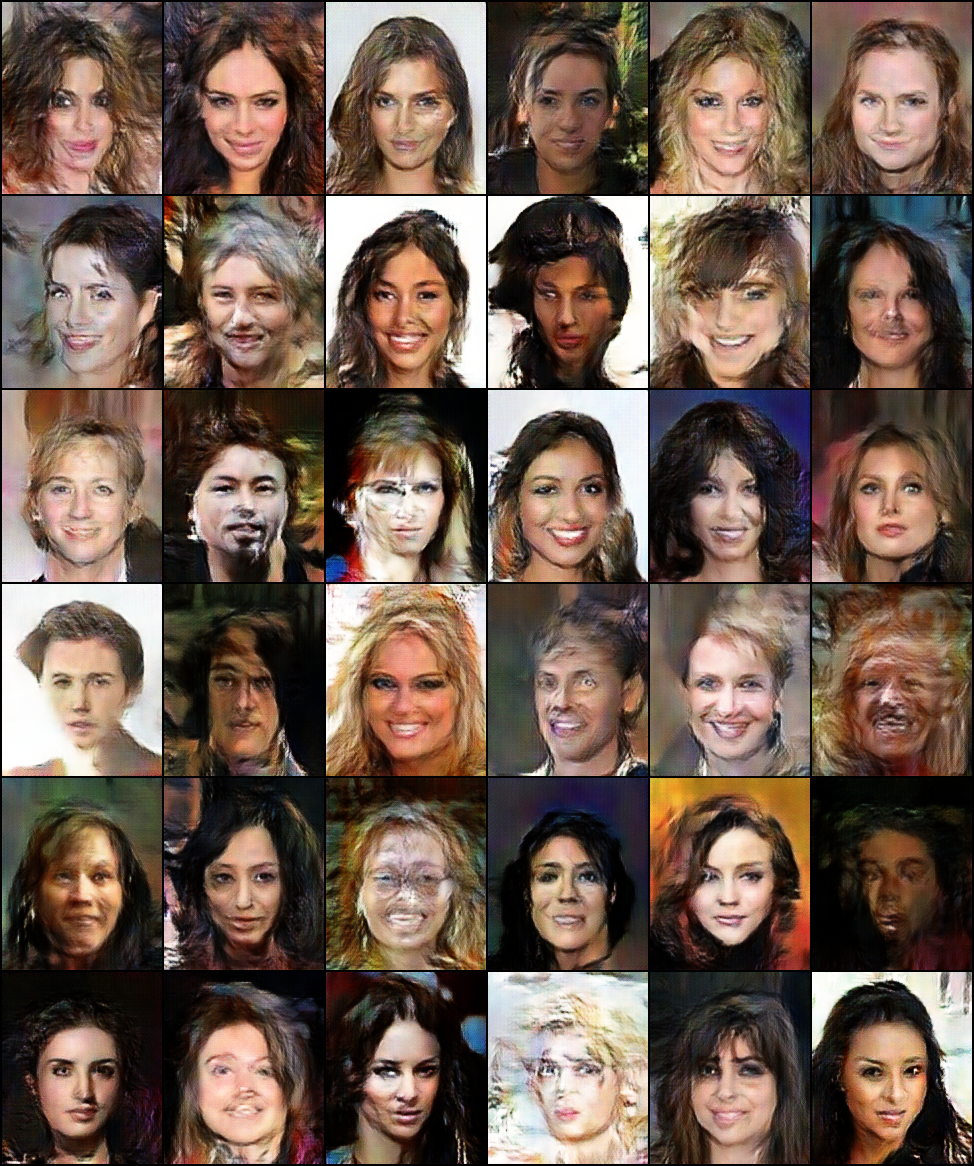} ~~~\includegraphics[scale=0.18]{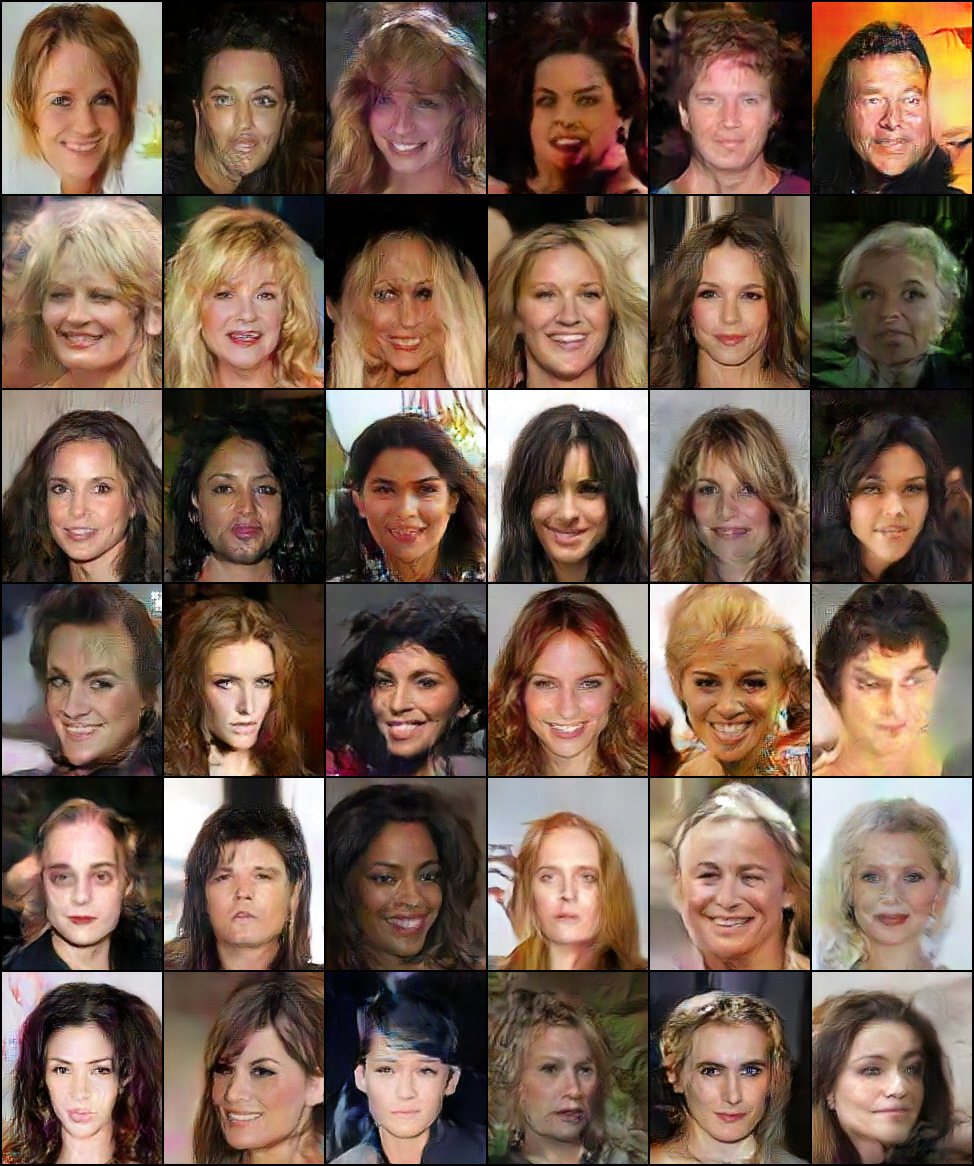}

\caption{Result of two GANs on CelebA. (Left) with Nonlinearity Normalization
and (Right) with Batch Normalization. In both cases the discriminator
uses a Normalized ReLU.}
\end{figure}

\section{FC-NN Order and Chaos}

In this section, we prove the existence of two regimes,`order' and
`chaos', in FC-NNs. First, we improve some results of \cite{Daniely},
and study the rate of convergence of the activation kernels as the
depth grows to infinity. In a second step, this allows us to characterise
the behavior of the NTK for large depth. 

Let us consider a standardized differentiable nonlinearity $\sigma$,
i.e. satisfying $\mathbb{E}_{x\sim\mathcal{N}\left(0,1\right)}\left[\sigma^{2}\left(x\right)\right]=1$.
Recall that the the activation kernels are defined recursively by
$\Sigma^{(1)}(x,y)=\frac{1-\beta^{2}}{n_{0}}x^{T}y+\beta^{2}$ and
$\Sigma^{(\ell+1)}(x,y)=(1-\beta^{2})\mathbb{L}_{\Sigma^{(L)}}^{\sigma}(x,y)+\beta^{2}$,
where $\mathbb{L}_{\Sigma^{(L)}}^{\sigma}$ was introduced in Section
2.2. By induction, for any $x,y\in\mathbb{S}_{n_{0}}$, $\Sigma^{(\ell+1)}(x,y)$
is uniquely determined by $\rho_{x,y}=\frac{1}{n_{0}}x^{T}y$. Defining
the two functions $R_{\sigma},B_{\beta}:[-1,1]\to[-1,1]$ by:
\begin{align*}
R_{\sigma}(\rho) & =\mathbb{E}_{v\sim\mathcal{N}\left(0,\left(\begin{array}{cc}
1 & \rho\\
\rho & 1
\end{array}\right)\right)}\left[\sigma(v_{0})\sigma(v_{1})\right],\\
B_{\beta}(\rho) & =\beta^{2}+(1-\beta^{2})\rho,
\end{align*}
one can formulate the activation kernels as an alternate composition
of $B_{\beta}$ and $R_{\sigma}$: 
\[
\Sigma^{(\ell)}(x,y)=\left(B_{\beta}\circ R_{\sigma}\right)^{\circ\ell-1}\circ B_{\beta}\left(\rho_{x,y}\right).
\]
In particular, this shows that for any $x,y\in\mathbb{S}_{n_{0}}$,
$\Sigma^{(\ell)}(x,y)\leq1$. Since the activation kernels are obtained
by iterating the same function, we first study the fixed points of
the composition $B_{\beta}\circ R_{\sigma}:[-1,1]\to[-1,1]$. When
$\sigma$ is a standardized nonlinearity, the function $R_{\sigma}$,
named the dual of $\sigma$, satisfies the following key properties
proven in \cite{Daniely}: 
\begin{enumerate}
\item $R_{\sigma}(1)=1$, 
\item For any $\rho\in(-1,0)$, $R_{\sigma}(\rho)>\rho$, 
\item $R_{\sigma}$ is convex in $[0,1)$,
\item $R_{\sigma}'(1)=\mathbb{E}\left[\dot{\sigma}(x)^{2}\right]$ , where
$R_{\sigma}'$ denotes the derivative of $R_{\sigma}$, 
\item $R_{\sigma}'=R_{\dot{\sigma}}$ .
\end{enumerate}
By definition $B_{\beta}(1)=1$, thus $1$ is a trivial fixed point:
$B_{\beta}\circ R_{\sigma}(1)=1$. This shows that for any $x\in\mathbb{S}_{n_{0}}$
and any $\ell\geq1$:
\[
\Sigma^{(\ell)}(x,x)=1.
\]

It appears that $-1$ is also a fixed point of $B_{\beta}\circ R_{\sigma}$
if and only if the nonlinearity $\sigma$ is antisymmetric and $\beta=0$.
From now on, we will focus on the region $(-1,1)$. From the property
2. of $R_{\sigma}$ and since $B_{\beta}$ is non decreasing, any
non trivial fixed point must lie in $[0,1)$. Since $B_{\beta}\circ R_{\sigma}(0)>0$,
$B_{\beta}\circ R_{\sigma}(1)=1$ and $R_{\sigma}$ is convex in $[0,1)$,
there exists a non trivial fixed point of $B_{\beta}\circ R_{\sigma}$
if $\left(B_{\beta}\circ R_{\sigma}\right)'(1)>1$ whereas if $\left(B_{\beta}\circ R_{\sigma}\right)'(1)<1$
there is no fixed point in $(-1,1)$. This leads to two regimes shown
in \cite{Daniely}, depending on the value of $r_{\sigma,\beta}=\left(1-\beta^{2}\right)\mathbb{E}_{x\sim\mathcal{N}\left(0,1\right)}\left[\dot{\sigma}^{2}\left(x\right)\right]$:
\begin{enumerate}
\item ``Order'' when $r_{\sigma,\beta}<1$: $B_{\beta}\circ R_{\sigma}$
has a unique fixed point equal to $1$ and the activation kernels
become constant at an exponential rate, 
\item ``Chaos'' when $r_{\sigma,\beta}>1$: $B_{\beta}\circ R_{\sigma}$
has another fixed point $0\leq a<1$ and the activation kernels converge
to a kernel equal to $1$ if $x=y$ and to $a$ if $x\neq y$ and,
if the nonlinearity is antisymmetric and $\beta=0$, it converges
to $-1$ if and only if $x=-y$.
\end{enumerate}
To establish the existence of the two regimes for the NTK, we need
the following bounds on the rate of convergence of $\Sigma^{(\ell)}(x,y)$
in the ``order'' region and on its values in the ``chaos'' region:
\begin{lem}
\label{prop:infinite_depth_Sigma}If $\sigma$ is a standardized differentiable
nonlinearity,

If $r_{\sigma,\beta}<1$, then for any $x,y\in\mathbb{S}_{n_{0}}$,
\[
1\geq\Sigma^{(\ell)}(x,y)\geq1-2r_{\sigma,\beta}^{\ell-1}(1-\beta^{2}).
\]
If $r_{\sigma,\beta}>1$, then there exists a fixed point $a\in[0,1)$
of \textup{$B_{\beta}\circ R_{\sigma}$} such that for any $x,y\in\mathbb{S}_{n_{0}}$,
\begin{align*}
\left|\Sigma^{(\ell)}(x,y)\right| & \leq\max\left\{ \left|\beta^{2}+\frac{1-\beta^{2}}{n_{0}}x^{T}y\right|,a\right\} .
\end{align*}
\end{lem}

\begin{proof}
Let us denote $r=r_{\sigma,\beta}$ and suppose first that $r<1$.
By \cite{Daniely}, we know that $R_{\sigma}'=R_{\dot{\sigma}}$ and
$R_{\dot{\sigma}}(\rho)\in\left[-\mathbb{E}\left[\dot{\sigma}(z)^{2}\right],\mathbb{E}\left[\dot{\sigma}(z)^{2}\right]\right]$
where $z\sim\mathcal{N}(0,1)$. From now on, we will omit to specify
the distribution asumption on $z$. The previous equalities and inequalities
imply that $R_{\sigma}(\rho)\geq1-\mathbb{E}\left[\dot{\sigma}(v)^{2}\right](1-\rho)$,
thus we obtain:
\begin{align*}
B_{\beta}\circ R_{\sigma}(\rho) & \geq\beta^{2}+(1-\beta^{2})(1-\mathbb{E}\left[\dot{\sigma}(z)^{2}\right](1-\rho))\\
 & =1-(1-\beta^{2})\mathbb{E}\left[\dot{\sigma}(z)^{2}\right](1-\rho)\\
 & =1-r(1-\rho).
\end{align*}
By definition, we then have $\Sigma^{(\ell)}(x,y)=\left(B_{\beta}\circ R_{\sigma}\right)^{\circ\ell-1}\circ B_{\beta}\left(\frac{1}{n_{0}}x^{T}y\right)\geq1-2(1-\beta^{2})r^{\ell-1}$.
Using the bound $\Sigma^{(\ell)}(x,y)\leq1$, this proves the first
assertion. 

When $r>1$, there exists a fixed point $a$ of $B_{\beta}\circ R_{\sigma}$
in $[0,1).$ By a convexity argument, for any $\rho$ in $[a,1)$,
$a\leq B_{\beta}\circ R_{\sigma}(\rho)\leq\rho$ and because $R_{\sigma}(\rho)$
is increasing in $[0,1)$, for all $\rho\in[0,a]$, $0\leq B_{\beta}\circ R_{\sigma}(\rho)\leq a$. 

For negative $\rho$, we claim that $\left|B_{\beta}\circ R_{\sigma}(\rho)\right|\leq B_{\beta}\circ R_{\sigma}(\left|\rho\right|),$which
entails the second assertion. Since $R_{\sigma}(\rho)=\sum_{i=0}^{\infty}b_{i}\rho^{i}$
for positive $b_{i}$s \cite{Daniely}, and the composition $B_{\beta}\circ R_{\sigma}(\rho)=\sum_{i=0}^{\infty}c_{i}\rho^{i}$
for $c_{0}=b_{0}(1-\beta^{2})+\beta^{2}\geq0$ and $c_{i}=b_{i}(1-\beta^{2})\geq0$
when $i>0$, we have
\[
\left|B_{\beta}\circ R_{\sigma}(\rho)\right|=\left|\sum_{i=0}^{\infty}c_{i}\rho^{i}\right|\leq\sum_{i=0}^{\infty}c_{i}\left|\rho\right|^{i}=B_{\beta}\circ R_{\sigma}(\left|\rho\right|).
\]
This leads to the inequality in the chaos regime. 
\end{proof}
Before studying the normalized NTK, let us remark that the NTK on
the diagonal (with $x=y$ in $\mathbb{S}_{n_{0}}$) is equal to: 
\begin{align*}
\Theta_{\infty}^{(L)}(x,x) & =\sum_{\ell=1}^{L}\Sigma^{(\ell)}(x,x)\prod_{k=\ell+1}^{L}\dot{\Sigma}^{(k)}(x,x)\\
 & =\sum_{\ell=1}^{L}\left((1-\beta^{2})\mathbb{E}\left[\dot{\sigma}(x)^{2}\right]\right)^{L-\ell}\\
 & =\frac{1-r^{L}}{1-r}.
\end{align*}
This shows that in the ``order'' regime, $\Theta_{\infty}^{(L)}(x,x)\overset{L\to\infty}{\longrightarrow}\nicefrac{1}{1-r}$
and in the ``chaos'' regime $\Theta_{\infty}^{(L)}(x,x)$ grows
exponentially. At the transition, $r=1$ and thus $\Theta_{\infty}^{(L)}(x,x)=L$.
Besides, if $x,y\in\mathbb{S}_{n_{0}}$, using the Cauchy-Schwarz
inequality, for any $\ell$ , $\left|\Sigma^{(\ell)}(x,y)\right|\leq\left|\Sigma^{(\ell)}(x,x)\right|$
and $\left|\dot{\Sigma}^{(\ell+1)}(x,y)\right|\leq\left|\dot{\Sigma}^{(\ell+1)}(x,x)\right|$.
This implies the following inequality: $\Theta_{\infty}^{(L)}(x,y)\leq\Theta_{\infty}^{(L)}(x,x)$.

We now study the normalized NTK $\vartheta_{L}\left(x,y\right)=\frac{\Theta_{\infty}^{(L)}(x,y)}{\Theta_{\infty}^{(L)}(x,x)}\leq1$.
\begin{thm}
\label{thm:infinite-depth-fc-nn-1}Suppose that $\sigma$ is twice
differentiable and standardized.

If $r<1$, we are in the ordered regime: there exists $C_{1}$ such
that for $x,y\in\mathbb{S}_{n_{0}}$,
\[
1-C_{1}Lr^{L}\leq\vartheta^{(L)}\left(x,y\right)\leq1.
\]
If $r>1$, we are in the chaotic regime: for $x\neq y$ in $\mathbb{S}_{n_{0}}$,
there exist $s<1$ and $C_{2}$, such that 
\[
\left|\vartheta^{(L)}\left(x,y\right)\right|\leq C_{2}s^{L}.
\]
\end{thm}

\begin{proof}
First, let us suppose that $r<1$. Recall that the NTK is defined
as 
\[
\Theta_{\infty}^{(L)}(x,y)=\sum_{\ell=1}^{L}\Sigma^{(\ell)}(x,y)\dot{\Sigma}^{(\ell+1)}(x,y)\ldots\dot{\Sigma}^{(L)}(x,y).
\]

For all $\ell$, $\Sigma^{(\ell)}(x,y)\leq\Sigma^{(\ell)}(x,x)=1$
and $\dot{\Sigma}^{(\ell)}(x,y)\leq\dot{\Sigma}^{(\ell)}(x,x)=r$.
Writing $\Sigma^{(\ell)}(x,y)=1-\epsilon^{(\ell)}$ and $\dot{\Sigma}^{(\ell)}(x,y)=r-\dot{\epsilon}^{(\ell)}$
for $\epsilon^{(\ell)},\dot{\epsilon}^{(\ell)}\geq0$ we have
\begin{align*}
\Theta_{\infty}^{(L)}(x,y) & =\sum_{\ell=1}^{L}\left(1-\epsilon^{(\ell)}\right)\prod_{k=\ell+1}^{L}r-\dot{\epsilon}^{(\ell)}\\
 & \leq\sum_{\ell=1}^{L}r^{L-\ell}-r^{L-\ell}\epsilon^{(\ell)}-\sum_{k=\ell+1}^{L}r^{L-\ell-1}\dot{\epsilon}^{(\ell)}
\end{align*}

Using the bound of Lemma \ref{prop:infinite_depth_Sigma} and the
fact that for any $x,y\in\mathbb{S}_{n_{0}}$, $\dot{\Sigma}^{(\ell)}(x,y)=(1-\beta^{2})R_{\dot{\sigma}}(\Sigma^{(\ell-1)}(x,y))\geq r-\psi\epsilon^{(\ell-1)}$
for $\psi=(1-\beta^{2})\mathbb{E}_{z\sim\mathcal{N}(0,1)}\left[\ddot{\sigma}(z)\right]$,
we obtain $\epsilon^{(\ell)}<2(1-\beta^{2})r^{\ell-1}$ and $\dot{\epsilon}^{(\ell)}\leq2(1-\beta^{2})\psi r^{\ell-2}$.
As a result: 
\begin{align*}
\Theta_{\infty}^{(L)}(x,y) & \geq\sum_{\ell=1}^{L}r^{L-\ell}-2(1-\beta^{2})r^{L-\ell}r^{\ell-1}-\sum_{k=\ell+1}^{L}2(1-\beta^{2})\psi r^{L-\ell-1}r^{k-2}\\
 & =\Theta_{\infty}^{(L)}(x,x)-2(1-\beta^{2})\sum_{\ell=1}^{L}r^{L-1}+\psi\sum_{k=\ell+1}^{L}r^{L-\ell+k-3}\\
 & =\Theta_{\infty}^{(L)}(x,x)-2(1-\beta^{2})\left[Lr^{L-1}+\psi\sum_{\ell=1}^{L}\sum_{k=0}^{L-\ell-1}r^{L+k-2}\right]\\
 & =\Theta_{\infty}^{(L)}(x,x)-2(1-\beta^{2})\left[Lr^{L-1}+\psi r^{L-2}\sum_{\ell=1}^{L}\frac{1-r^{L-\ell}}{1-r}\right]\\
 & \geq\Theta_{\infty}^{(L)}(x,x)-2(1-\beta^{2})\left[r+\psi\frac{1}{1-r}\right]Lr^{L-2}\\
 & \geq\Theta_{\infty}^{(L)}(x,x)-CLr^{L}.
\end{align*}

Now, let us suppose that $r>1$. Recall that $B_{\beta}\circ R_{\sigma}$
has a unique fixed point $a$ on $[0,1).$ For any $x$ and $y$ in
$\mathbb{S}_{n_{0}}$, the kernels $\Sigma^{(\ell)}(x,y)$ are bounded
in norm by $v=max\left\{ \left|\beta^{2}+\frac{1-\beta^{2}}{n_{0}}x^{T}y\right|,a\right\} $
from Lemma \ref{prop:infinite_depth_Sigma}. For the kernels $\dot{\Sigma}^{(\ell)}$
we have $\left|\dot{\Sigma}^{(\ell)}(x,y)\right|=(1-\beta^{2})\left|R_{\dot{\sigma}}(\Sigma^{(\ell-1)}(x,y))\right|\leq(1-\beta^{2})R_{\dot{\sigma}}(\left|\Sigma^{(\ell-1)}(x,y)\right|)\leq(1-\beta^{2})R_{\dot{\sigma}}(v)=:w$
where the first inequality follows from the fact that $R_{\dot{\sigma}}(\rho)=\sum_{i}b_{i}\rho^{i}$
for $b_{i}\geq0$ and the second follows from the monotonicity of
$R_{\dot{\sigma}}$ in $[0,1]$. Applying these two bounds, we obtain:
\[
\left|\Theta_{\infty}^{(L)}(x,y)\right|\leq\sum_{\ell=1}^{L}v\prod_{k=\ell+1}^{L}w=v\frac{1-w^{L}}{1-w}.
\]
Since $\Theta_{\infty}^{(L)}(x,y)=\frac{1-r^{L}}{1-r}$, we have that
$\left|\vartheta_{L}\left(x,y\right)\right|\leq v\frac{1-w^{L}}{1-r^{L}}$.
If $x\neq y$ then $v<1$ and since $\sigma$ is nonlinear, $w=(1-\beta^{2})R_{\dot{\sigma}}(v)<(1-\beta^{2})R_{\dot{\sigma}}(1)=r$.
This implies that $\left|\vartheta_{L}\left(x,y\right)\right|$ converges
to zero at an exponential rate, as $L\to\infty$. 
\end{proof}

\subsection{ReLU FC-NN \label{subsec:ReLU-FC-NN}}

For the standardized ReLU nonlinearity, $\sigma\left(x\right)=\sqrt{2}\max\left(x,0\right)$,
the dual activation is computed in \cite{Daniely}:
\[
R_{\sigma}(\rho)=\frac{\sqrt{1-\rho^{2}}+\left(\pi-\cos^{-1}(\rho)\right)\rho}{\pi},
\]
and the dual activation of its derivative is given by:
\[
R_{\dot{\sigma}}(\rho)=\frac{\pi-\cos^{-1}(\rho)}{\pi}.
\]

The characteristic value $r=r_{\sigma,\beta}$ of the standardized
ReLU is equal to $1-\beta^{2}$: the ReLU nonlinearity therefore lies
in the ``order'' regime as soon as $\beta>0$. More explicitly,
Lemma \ref{prop:infinite_depth_Sigma} still holds of the standardized
ReLU and, using the value of $r$, the following inequalities hold
for any $x,y\in\mathbb{S}_{n_{0}}$: 
\[
1\geq\Sigma^{(\ell)}(x,y)\geq1-2r^{\ell}.
\]
Using these bounds, we can now prove Theorem \ref{thm:infinite-relu}.
\begin{thm}
\label{thm:infinite-relu}With the same notation as in Theorem \ref{thm:infinite-depth-fc-nn-1},
taking $\sigma$ to be the standardized ReLU and $\beta>0$, we are
in the weakly ordered regime: there exists a constant $C$ such that
$1-CLr^{L/2}\leq\vartheta^{(L)}\left(x,y\right)\leq1$. 
\end{thm}

\begin{proof}
The first inequality $\vartheta_{L}\left(x,y\right)\leq1$ follows
the same proof as in the differentiable case. 

For the lower bound, using the fact that $\left(1-\beta\right)r=1$,
we have $\epsilon^{(\ell)}=1-\Sigma^{(\ell)}(x,y)\leq2r^{\ell}$ and
using the explicit value of $R_{\dot{\sigma}}(\rho)$, we get that
$R_{\dot{\sigma}}(\rho)\geq1-\sqrt{1-\rho}$ which implies that $\dot{\epsilon}^{(\ell)}=r-\dot{\Sigma}^{(\ell)}(x,y)\leq r\sqrt{2}r^{\frac{\ell-1}{2}}$:
\begin{align*}
\Theta_{\infty}^{(L)}(x,y) & =\sum_{\ell=1}^{L}\left(1-\epsilon^{(\ell)}\right)\prod_{k=\ell+1}^{L}r-\dot{\epsilon}^{(k)}\\
 & \geq\sum_{\ell=1}^{L}r^{L-\ell}-2r^{L-\ell}r^{\ell}-\sqrt{2}\sum_{k=\ell+1}^{L}r^{L-\ell-1+\frac{k-1}{2}}\\
 & =\Theta_{\infty}^{(L)}(x,x)-2Lr^{L}-\sqrt{2}\sum_{\ell=1}^{L}r^{L-\frac{\ell}{2}-1}\sum_{k=0}^{L-\ell-1}r^{\frac{k}{2}}\\
 & \geq\Theta_{\infty}^{(L)}(x,x)-2Lr^{L}-\frac{\sqrt{2}}{1-\sqrt{r}}r^{\nicefrac{L}{2}-1}\sum_{\ell=0}^{L-1}r^{\nicefrac{\ell}{2}}\\
 & =\Theta_{\infty}^{(L)}(x,x)-2Lr^{L}-\frac{\sqrt{2}}{1-\sqrt{r}}r^{\nicefrac{L}{2}-1}\frac{1}{1-\sqrt{r}}\\
 & \geq\Theta_{\infty}^{(L)}(x,x)-2Lr^{L}-\frac{\sqrt{2}}{\left(1-\sqrt{r}\right)^{2}}r^{\nicefrac{L}{2}-1}\\
 & \geq\Theta_{\infty}^{(L)}(x,x)-\left[2Lr^{\nicefrac{L}{2}}-\frac{\sqrt{2}}{r\left(1-\sqrt{r}\right)^{2}}\right]r^{\nicefrac{L}{2}}\\
 & \geq\Theta_{\infty}^{(L)}(x,x)-C_{0}(r,\beta)r^{\nicefrac{L}{2}}.
\end{align*}
Recall that for any $x\in\mathbb{S}_{n_{0}}$, $\Theta_{\infty}^{(L)}(x,x)=\frac{1-r^{L}}{1-r}$
is bounded in $L$. Dividing the previous inequality by $\Theta_{\infty}^{(L)}(x,x)$
we get: $1-Cr^{L/2}\leq\vartheta_{L}\left(x,y\right)\leq1.$
\end{proof}

\section{Layer Normalization and Nonlinearity Normalization}

\subsection{Layer normalization asymptotically equivalent to nonlinearity normalization.\protect \\
}

With Layer Normalization (LN), the coordinates of the normalized vectors
of activations are $\check{\alpha}_{j}^{(\ell)}(x)=\sqrt{n_{\ell}}\frac{\alpha_{j}^{(\ell)}(x)-\mu^{(\ell)}(x)}{||\alpha^{(\ell)}(x)-\underline{\mu}^{(\ell)}(x)||}$,
where $\mu^{(\ell)}:=\frac{1}{n_{\ell}}\sum_{i=1}^{n_{\ell}}\alpha_{i}^{(\ell)}(x)$
and $\underline{\mu}^{(\ell)}:=\begin{pmatrix}\mu^{(\ell)}\\
\vdots\\
\mu^{(\ell)}
\end{pmatrix}$. We simplify the notation by keeping the dependence on $x$ implicite
and denote the standardized nonlinearity $\underline{\sigma}(\cdot):=\frac{\sigma(\cdot)-\mathbb{E}(\sigma(Z))}{\sqrt{\mathrm{Var}(\sigma(Z))}}$,
where $Z\stackrel{d}{\sim}\mathcal{N}(0,1)$. 

Suppose that $L=2$, that is we have a single hidden layer after which
the LN is applied. More precisely, the output of the network function
with LN is $\widetilde{\alpha}^{(2)}(\check{\alpha}^{(1)}(x))$. We
rewrite 
\begin{align*}
 & \check{\alpha}^{(1)}=\sqrt{n_{1}}\frac{\sigma(\widetilde{\alpha}^{(1)})-\underline{\mu}^{(1)}}{||\sigma(\widetilde{\alpha}^{(1)})-\underline{\mu}^{(1)}||}=\underline{\sigma}(\widetilde{\alpha}^{(1)})C_{1}+C_{2},\\
\text{where}\qquad & C_{1}=\sqrt{n_{1}}\frac{\mathrm{\sqrt{Var(\sigma(Z))}}}{||\sigma(\widetilde{\alpha}^{(1)})-\underline{\mu}^{(1)}||},\quad\text{and}\quad C_{2}=\sqrt{n_{1}}\frac{\mathbb{E}(\sigma(Z))-\mu^{(1)}}{||\sigma(\widetilde{\alpha}^{(1)})-\underline{\mu}^{(1)}||}.
\end{align*}
Note that $C_{1}\to1$ and $C_{2}\to0$ almost surely, as $n_{1}\to\infty$.
Indeed, since the $\widetilde{\alpha}_{i}^{(1)}$'s are independent
standard Gaussian variables at initialization (recall that we assume
that the inputs belong to $\mathbb{S}_{n_{0}}$), the law of large
numbers entails that $\mu^{(1)}\to\mathbb{E}(\sigma(Z))$ almost surely,
as $n_{1}\to\infty$, and similarly for $\frac{||\sigma(\widetilde{\alpha}^{(1)})-\underline{\mu}^{(1)}||^{2}}{n_{1}}\to\mathrm{Var}(\sigma(Z))$.

To show that LN is asymptotically equivalent to centering and standardizing
the nonlinearity, we now establish that $C_{1}$ and $C_{2}$ are
constant during training. We have 
\begin{align}
\frac{\partial}{\partial\widetilde{\alpha}_{j}^{(1)}}||\sigma(\widetilde{\alpha}^{(1)})-\underline{\mu}^{(1)}||=\frac{\dot{\sigma}(\widetilde{\alpha}_{j}^{(1)})\sum_{i=1}^{n_{1}}(\delta_{ij}-1/n_{1})(\sigma(\widetilde{\alpha}_{i}^{(1)})-\mu^{(1)})}{||\sigma(\widetilde{\alpha}^{(1)})-\underline{\mu}^{(1)}||}=\frac{\dot{\sigma}(\widetilde{\alpha}_{j}^{(1)})(\sigma(\widetilde{\alpha}_{j}^{(1)})-\mu^{(1)})}{||\sigma(\widetilde{\alpha}^{(1)})-\underline{\mu}^{(1)}||}.\label{Equation norm derivative}
\end{align}
Note that the absolute value of the latter is bounded by $2||\dot{\sigma}||_{\infty}$.
We write $g(t)$ for any function $g$ that depends on the parameters
$\theta(t)$ at time $t\geq0$. Using twice the triangle inequality
yields that 
\begin{align}
 & \Big|||\sigma(\widetilde{\alpha}^{(1)}(t))-\underline{\mu}^{(1)}(t)||-||\sigma(\widetilde{\alpha}^{(1)}(0))-\underline{\mu}^{(1)}(0)||\Big|\leq||\sigma(\widetilde{\alpha}^{(1)}(t))-\sigma(\widetilde{\alpha}^{(1)}(0))||+||\underline{\mu}^{(1)}(t)-\underline{\mu}^{(1)}(0)||\nonumber \\
 & \hspace{1.5cm}\leq||\dot{\sigma}||_{\infty}\left(\left(\sum_{i=1}^{n_{1}}(\widetilde{\alpha}_{i}^{(1)}(t)-\widetilde{\alpha}_{i}^{(1)}(0))^{2}\right)^{1/2}+\frac{1}{\sqrt{n_{1}}}\sum_{i=1}^{n_{1}}\left|\widetilde{\alpha}_{i}^{(1)}(t)-\widetilde{\alpha}_{i}^{(1)}(0)\right|\right)\leq ct,\label{Equation bound variation norm}
\end{align}
for some constant $c>0$, where we used that $|\widetilde{\alpha}_{i}^{(1)}(t)-\widetilde{\alpha}_{i}^{(1)}(0)|=\mathcal{O}(t/\sqrt{n_{1}})$,
see Appendix A.2 of \cite{jacot2018neural}. Since $||\sigma(\widetilde{\alpha}^{(1)}(0))-\underline{\mu}^{(1)}(0)||\sim\sqrt{n_{1}}$
by the law of large numbers, we can always write $||\sigma(\widetilde{\alpha}^{(1)}(t))-\underline{\mu}^{(1)}(t)||>||\sigma(\widetilde{\alpha}^{(1)}(0))-\underline{\mu}^{(1)}(0)||-ct>0$.
Hence, using (\ref{Equation norm derivative}) then (\ref{Equation bound variation norm}),
we get 
\begin{align}
\left|\frac{\partial C_{1}(t)}{\partial\widetilde{\alpha}_{j}^{(1)}(t)}\right| & =\frac{\sqrt{n_{1}}\mathrm{Var}(\sigma(Z))}{||\sigma(\widetilde{\alpha}^{(1)}(t))-\underline{\mu}^{(1)}(t)||^{2}}\cdot\left|\frac{\dot{\sigma}(\widetilde{\alpha}_{j}^{(1)}(t))(\sigma(\widetilde{\alpha}_{j}^{(1)}(t))-\mu^{(1)}(t))}{||\sigma(\widetilde{\alpha}^{(1)}(t))-\underline{\mu}^{(1)}(t)||}\right|\nonumber \\
 & \leq\frac{\sqrt{n_{1}}\mathrm{Var}(\sigma(Z))}{(||\sigma(\widetilde{\alpha}^{(1)}(0)-\underline{\mu}^{(1)}(0))||-ct)^{2}}||\dot{\sigma}||_{\infty}=\mathcal{O}(1/\sqrt{n_{1}}),\label{Equation upper bound derivative C_1}
\end{align}
by the law of large numbers. The case of $C_{2}$ is similar: 
\begin{align}
\frac{\partial C_{2}(t)}{\partial\widetilde{\alpha}_{j}^{(1)}(t)} & =\frac{-\dot{\sigma}(\widetilde{\alpha}_{j}^{(1)}(t))}{\sqrt{n_{1}}||\sigma(\widetilde{\alpha}^{(1)}(t))-\underline{\mu}^{(1)}(t)||}-\sqrt{n_{1}}\frac{(\mathbb{E}(\sigma(Z))-\mu^{(1)}(t))\dot{\sigma}(\widetilde{\alpha}_{j}^{(1)}(t))(\sigma(\widetilde{\alpha}_{j}^{(1)}(t))-\mu^{(1)}(t))}{||\sigma(\widetilde{\alpha}^{(1)}(t))-\underline{\mu}^{(1)}(t)||^{3}}\nonumber \\
 & \leq||\dot{\sigma}||_{\infty}\left(\frac{1}{n_{1}}\frac{\sqrt{n_{1}}}{||\sigma(\widetilde{\alpha}^{(1)}(0))-\underline{\mu}^{(1)}(0)||-ct}-\frac{1}{\sqrt{n_{1}}}\frac{n_{1}(\mathbb{E}(\sigma(Z))-\mu^{(1)}(0)+ct)}{(||\sigma(\widetilde{\alpha}^{(1)}(0))-\underline{\mu}^{(1)}(0)||-ct)^{2}}\right)=\mathcal{O}(1/\sqrt{n_{1}}),\label{Equation upper bound derivative C_2}
\end{align}
again by the law of large numbers. For $i=1,2$, we now write $\frac{\partial C_{i}(t)}{\partial t}=\frac{\partial\widetilde{\alpha}_{j}^{(1)}(t)}{\partial t}\frac{\partial C_{i}(t)}{\partial\widetilde{\alpha}_{j}^{(1)}(t)}$
and recall that the first term is changing at rate $\mathcal{O}(1/\sqrt{n_{1}})$.
Therefore, $|C_{i}(t)-C_{i}(0)|\leq\mathcal{O}(t/n_{1})$. The claim
for $L\geq3$ follows by induction.

\subsection{Pre-layer normalization has asymptotically no effect.\protect \\
}

Normalizing the preactivations has asymptotically no effect on the
network at initialization as well as during training. The output of
the $\ell$-th layer becomes $\check{\alpha}_{j}^{(\ell)}=\sigma\big(\sqrt{n_{\ell}}\frac{\widetilde{\alpha}_{j}^{(\ell)}-\mu^{(\ell)}}{||\widetilde{\alpha}^{(\ell)}-\underline{\mu}^{(\ell)}||}\big)$
where $\mu^{(\ell)}$ and $\underline{\mu}^{(\ell)}$ are computed
similarily as before with $\widetilde{\alpha}^{(\ell)}$ in place
of $\alpha^{(\ell)}$. As before, we assume $L=2$ and deduce the
general case by induction. We write $\check{\alpha}_{j}^{(1)}=\sigma(\widetilde{\alpha}_{j}^{(1)}C_{1}+C_{2})$,
with $C_{1}=\sqrt{n_{1}}/||\widetilde{\alpha}^{(\ell)}-\underline{\mu}^{(\ell)}||$
and $C_{2}=-\sqrt{n_{1}}\mu^{(1)}/||\widetilde{\alpha}^{(\ell)}-\underline{\mu}^{(\ell)}||$.
Again, the law of large numbers show that $C_{1}\to1$ and $C_{2}\to0$
almost surely, as $n_{1}\to\infty$. Moreover, similarily as (\ref{Equation norm derivative})
and (\ref{Equation bound variation norm}), we have that 
\begin{align*}
 & \frac{\partial}{\partial\widetilde{\alpha}_{j}^{(1)}}||\widetilde{\alpha}^{(1)}-\underline{\mu}^{(1)}||=\frac{\widetilde{\alpha}_{j}^{(1)}-\mu^{(1)}}{||\widetilde{\alpha}^{(1)}-\underline{\mu}^{(1)}||},\\
 & \Big|||\widetilde{\alpha}^{(1)}(t)-\underline{\mu}^{(1)}(t)||-||\widetilde{\alpha}^{(1)}(0)-\underline{\mu}^{(1)}(0)||\Big|\leq ct,
\end{align*}
for some constant $c>0$. Using the same argument as in (\ref{Equation upper bound derivative C_1})
and (\ref{Equation upper bound derivative C_2}), one can thus show
for $i=1,2$ that 
\begin{align*}
\left|\frac{\partial C_{i}(t)}{\partial\widetilde{\alpha}_{j}^{(1)}}\right|=\mathcal{O}(1/\sqrt{n_{1}}).
\end{align*}
We conclude as previously, noting that 
\begin{align*}
\frac{\partial\check{\alpha}_{j}^{(1)}(t)}{\partial t}=\dot{\sigma}\left(\widetilde{\alpha}_{j}^{(1)}(t)C_{1}(t)+C_{2}(t)\right)\left(\frac{\partial\widetilde{\alpha}_{j}^{(1)}(t)}{\partial t}C_{1}(t)+\widetilde{\alpha}_{j}^{(1)}(t)\frac{\partial C_{1}(t)}{\partial t}+\frac{\partial C_{2}(t)}{\partial t}\right).
\end{align*}

\section{Batch Normalization}

If one adds a BatchNorm layer after the nonlinearity of the last hidden
layer, we have:
\begin{lem}
Consider a FC-NN with $L$ layers, with a PN-BN after the last nonlinearity.
For any $k,k'\in\left\{ 1,\ldots,n_{L}\right\} $ and any parameter
$\theta_{p}$, we have $\sum_{i=1}^{N}\Theta_{\theta_{p}}^{\left(L\right)}\left(\cdot,x_{i}\right)=\beta^{2}\mathrm{Id}_{n_{L}}$. 
\end{lem}

\begin{proof}
This is an direct consequence of the definition of the NTK and of
the following claim: 
\begin{claim*}
For a fully-connected DNN with a BatchNorm layer after the nonlinearity
of the last hidden layer then $\frac{1}{N}\sum_{i=1}^{N}\partial_{\theta_{p}}f_{\theta,k}(x_{i})$
is equal to $\beta$ if $\theta_{p}$ is $b_{k}^{(L-1)}$, the bias
parameter of the last layer, and equal to $0$ otherwise. 
\end{claim*}
The average of $f_{\theta,k}$ on the training set, $\frac{1}{N}\sum_{i=1}^{N}\partial_{\theta_{p}}f_{\theta,k}(x_{i})$,
only depends on the bias of the last layer:
\[
\frac{1}{N}\sum_{i=1}^{N}f_{\theta,k}(x_{i})=\frac{\sqrt{1-\beta^{2}}}{\sqrt{n_{L-1}}}W^{(L-1)}\frac{1}{N}\sum_{i=1}^{N}\hat{\alpha}^{(L-1)}(x_{i})+\beta b_{k}^{(L-1)}=\beta b_{k}^{(L-1)}.
\]
Thus for any parameter $\theta_{p}$, $\frac{1}{N}\sum_{i=1}^{N}\partial_{\theta_{p}}f_{\theta,k}(x_{i})=\partial_{\theta_{p}}\left(\beta b_{k}^{(L-1)}\right)$
is equal to $\beta$ if the parameter is the bias $b_{k}^{(L-1)}$
and zero otherwise.
\end{proof}

\section{Graph-based Neural Networks\label{sec:General-Convolutional-Networks}}

In this section, we prove the convergence of the NTK at initialization
for a general family of DNNs which contain in particular CNNs and
DC-NNs. We will consider the Graph-based parametrization introduced
in the main. 

For each layer $\ell=0,...,L$, the neurons are indexed by a position
$p\in I_{\ell}$ and a channel $i=1,...,n_{\ell}$. We may assume
that the sets of positions $I_{\ell}$ can be any set, in particular
any subset of $\mathbb{Z}^{D}$. For any position $p\in I_{\ell+1}$,
we consider a set of parents $P(p)\subset I_{\ell}$ and we define
recursively the set $P^{\circ k}(p)\subset I_{\ell+1-k}$ of ancestors
of level $k$ by $P^{\circ k}(p)=\left\{ q\mid\exists q'\in P^{\circ k-1}(p),q\in P(q')\right\} $.
For each parent $q\in P(p)$, the connections from the position $\left(q,\ell\right)$
to the position $\left(p,\ell+1\right)$ are encoded in an $n_{\ell}\times n_{\ell+1}$
weight matrix $W^{(\ell,q\to p)}$. We define $\chi(q\to p,q'\to p')$
which is equal to $1$ if and only if $W^{(\ell,q\to p)}$ and $W^{(\ell,q'\to p')}$
are shared (in the sense that the two matrices are forced to be equal
at initialization and during training) and $0$ otherwise. It satisfies
$\chi(q\to p,q\to p)=1$ for any neuron $p$ and any $q\in P(p)$
and it is transitive. We will also suppose that for any neuron $p$
and any $q,q'\in P(p)$, $\chi(q\to p,q'\to p)=\delta_{qq'}$ (i.e.
no pair of connections connected to the same neuron $p$ are shared). 

In this setting, the activations and preactivations $\alpha^{\left(\ell\right)},\tilde{\alpha}^{\left(\ell\right)}\in\left(\mathbb{R}^{n_{\ell}}\right)^{I_{\ell}}$
are recursively constructed using the parent-based NTK parametrization:
we set $\alpha^{\left(0\right)}\left(x\right)=x$ and for $\ell=0,\ldots,L-1$
and any position $p\in I_{\ell}$:
\[
\tilde{\alpha}^{(\ell+1,p)}(x)=\beta b^{(\ell)}+\frac{\sqrt{1-\beta^{2}}}{\sqrt{\left|P(p)\right|n_{\ell}}}\sum_{q\in P(p)}W^{(\ell,q\to p)}x_{q},\qquad\alpha^{\left(\ell+1\right)}\left(x\right)=\sigma\left(\tilde{\alpha}^{\left(\ell+1\right)}\left(x\right)\right)
\]
where $\sigma$ is applied entry-wise, $\beta\geq0$ and $\left|P(p)\right|$
is the cardinal of $P(p)$. This is a slightly more general formalism
than the DC-NNs and it will allow us to obtain simpler formulae which
generalize well to other architectures. 
\begin{rem}
Notice that the parametrization is slightly different than the traditional
one: we divide by $\sqrt{\left|P(p)\right|n_{\ell}}$ instead of dividing
by $\sqrt{n_{\ell}\nicefrac{\left|\omega\right|}{s_{1}\ldots s_{d}}}$
. This does not lead to any difference when one consider infinite-sized
images as in Section \ref{sec:DC-NN-Freeze-and-Chaos} since in this
case the number of parents is constant, equal to $\nicefrac{\left|\omega\right|}{s_{1}\ldots s_{d}}$.
The key difference between the two parametrizations will be investigated
in Section \ref{sec:Border-Effects-appendix}. 
\end{rem}

Recall, that for a kernel $K:\mathbb{R}^{n_{0}}\times\mathbb{R}^{n_{0}}\to\mathbb{R}$,
and for any $z_{0},z_{1}\in\mathbb{R}^{n_{0}}$, we defined: 
\[
\mathbb{L}_{K}^{g}\left(z_{0},z_{1}\right)=\mathbb{E}_{\left(y_{0},y_{1}\right)\sim\mathcal{N}\left(0,\left(K\left(z_{i},z_{j}\right)\right)_{i,j=0,1}\right)}\left[g\left(y_{0}\right)g\left(y_{1}\right)\right].
\]
 
\begin{prop}
\label{prop:conv_Sigma_parents}In this setting, as $n_{1}\to\infty$,
$\ldots$,$n_{\ell-1}\to\infty$ sequentially, the preactivations
$\left(\tilde{\alpha}_{i}^{(\ell,p)}(x)\right)_{i=1,\ldots,n_{\ell},p\in I_{\ell}}$
of the $\ell^{th}$ layer converge to a centered Gaussian process
with covariance $\Sigma^{(\ell,pp')}(x,y)\delta_{ii'}$, where $\Sigma^{(\ell,pp')}(x,y)$
is defined recursively as
\begin{align*}
\Sigma^{(1,pp')}(x,y) & =\beta^{2}+\frac{1-\beta^{2}}{\sqrt{\left|P(p)\right|\left|P(p')\right|}n_{0}}\sum_{q\in P(p)}\sum_{q'\in P(p')}\chi(q\to p,q'\to p')\left(x_{q}\right)^{T}y_{q'},\\
\Sigma^{(\ell+1,pp')}(x,y) & =\beta^{2}+\frac{1-\beta^{2}}{\sqrt{\left|P(p)\right|\left|P(p')\right|}}\sum_{q\in P(p)}\sum_{q'\in P(p')}\chi(q\to p,q'\to p')\mathbb{L}_{\Sigma^{(\ell,qq')}}^{\sigma}\left(x,y\right).
\end{align*}
\end{prop}

\begin{proof}
The proof is done by induction on $\ell$. For $\ell=1$ and any $i\in\left\{ 1,\ldots,n_{1}\right\} $,
the preactivation
\[
\tilde{\alpha}_{i}^{(1,p)}(x)=\beta b_{i}^{(0)}+\frac{\sqrt{1-\beta^{2}}}{\sqrt{\left|P(p)\right|n_{0}}}\sum_{q\in P(p)}\left(W_{p}^{(0,q\to p)}x_{q}\right)_{i}
\]
is a random affine function of $x$ and its coefficients are centered
Gaussian: it is hence a centered Gaussian process whose covariance
is easily shown to be equal to $\mathbb{E}\left[\tilde{\alpha}_{i}^{(1,p)}(x)\tilde{\alpha}_{i'}^{(1,p')}(y)\right]=\Sigma^{(1,pp')}(x,y)\delta_{ii'}$. 

For the induction step, we assume that the result holds for the pre-activations
of the layer $\ell$. The pre-activations of the next layer are of
the form
\[
\tilde{\alpha}_{i}^{(\ell+1,p)}(x)=\beta b_{i}^{(0)}+\frac{\sqrt{1-\beta^{2}}}{\sqrt{\left|P(p)\right|n_{\ell}}}\sum_{q\in P(p)}\left(W^{(\ell,q\to p)}\alpha^{(\ell,q)}(x)\right)_{i}.
\]
Conditioned on the activations $\alpha^{(\ell,q)}$ of the last layer,
$\tilde{\alpha}^{(\ell+1,p)}$ is a centered Gaussian process: in
other terms, it is a mixture of centered Gaussians with a random covariance
determined by the activations of the last layer. The random covariance
between $\tilde{\alpha}_{i_{0}}^{(\ell+1,p_{0})}(x)$ and $\tilde{\alpha}_{i_{1}}^{(\ell+1,p_{1})}(y)$
is equal to
\begin{align*}
 & \beta^{2}\delta_{i_{0}i_{1}}+\frac{1-\beta^{2}}{\sqrt{\left|P(p)\right|\left|P(p')\right|}n_{\ell}}\sum_{\begin{array}{c}
q_{0}\in P(p_{0})\\
q_{1}\in P(p_{1})
\end{array}}\sum_{j_{0},j_{1}=1}^{n_{\ell}}\mathbb{E}\left[W_{i_{0}j_{0}}^{(\ell,q_{0}\to p_{0})}W_{i_{1}j_{1}}^{(\ell,q_{1}\to p_{1})}\right]\alpha_{j_{0}}^{(\ell,q_{0})}(x)\alpha_{j_{1}}^{(\ell,q_{1})}(y)\\
 & =\delta_{i_{0}i_{1}}\left[\beta^{2}+\frac{1-\beta^{2}}{\sqrt{\left|P(p)\right|\left|P(p')\right|}}\sum_{\begin{array}{c}
q_{0}\in P(p_{0})\\
q_{1}\in P(p_{1})
\end{array}}\chi(q_{0}\to p_{0},q_{1}\to p_{1})\frac{1}{n_{\ell}}\sum_{j=1}^{n_{\ell}}\sigma\left(\tilde{\alpha}_{j}^{(\ell,q_{0})}(x)\right)\sigma\left(\tilde{\alpha}_{j}^{(\ell,q_{1})}(y)\right)\right],
\end{align*}
where we used the fact that $\mathbb{E}\left[W_{i_{0}j_{0}}^{(\ell,q_{0}\to p_{0})}W_{i_{1}j_{1}}^{(\ell,q_{1}\to p_{1})}\right]=\chi(q_{0}\to p_{0},q_{1}\to p_{1})\delta_{i_{0}i_{1}}\delta_{j_{0}j_{1}}$.
Using the induction hypothesis, as $n_{1}\to\infty$, $\ldots$,$n_{\ell-1}\to\infty$
sequentially, the preactivations $\left(\tilde{\alpha}_{j}^{(\ell,q_{0})}(x),\tilde{\alpha}_{j}^{(\ell,q_{1})}(y)\right)_{j}$
converge to independant centered Gaussian pairs. As $n_{\ell}\to\infty$,
by the law of large numbers, the sum over $j$ along with the $\nicefrac{1}{n_{\ell}}$
converges to $\mathbb{L}_{\sigma}^{\Sigma^{(\ell,qq')}}\left(x,y\right)$.
In this limit, the random covariance of the Gaussian mixture becomes
deterministic and as a consequence, the mixture of Gaussian processes
tends to a centered Gaussian process with the right covariance.
\end{proof}
Similarly to the activation kernels, one can prove that the NTK converges
at initialization. 
\begin{prop}
\label{prop:convergence_NTK_parents}As $n_{1}\to\infty$, $\ldots$,$n_{L-1}\to\infty$
sequentially, the NTK $\Theta^{(L,p_{0}p_{1})}$ of a general convolutional
network converges to $\Theta_{\infty,p_{0}p_{1}}^{(L)}\otimes\mathrm{Id}_{n_{L}}$
where $\Theta_{\infty}^{(L,p_{0}p_{1})}(x,y)$ is defined recursively
by:
\begin{align*}
\Theta_{\infty}^{(1,p_{0}p_{1})}\left(x,y\right)= & \Sigma^{(1,p_{0}p_{1})}(x,y),\\
\Theta_{\infty}^{(L,p_{0}p_{1})}\left(x,y\right)= & \frac{1-\beta^{2}}{\sqrt{\left|P(p_{0})\right|\left|P(p_{1})\right|}}\sum_{\begin{array}{c}
q_{0}\in P(p_{0})\\
q_{1}\in P(p_{1})
\end{array}}\chi(q_{0}\to p_{0},q_{1}\to p_{1})\Theta_{\infty}^{(L-1,q_{0}q_{1})}(x,y)\mathbb{L}_{\Sigma^{(L-1,q_{0}q_{1})}}^{\dot{\sigma}}\left(x,y\right)\\
 & +\Sigma^{(L,p_{0}p_{1})}(x,y).
\end{align*}
\end{prop}

\begin{proof}
The proof by induction on $L$ follows the one of \cite{jacot2018neural}
for fully-connected DNNs. We present the induction step and assume
that the result holds for a general convolutional network with $L-1$
hidden layers. Following the same computations as in \cite{jacot2018neural},
the NTK $\Theta_{p_{0}p_{1},jj'}^{(L+1)}(x,y)$ is equal to 

\begin{align*}
 & \frac{1-\beta^{2}}{\sqrt{\left|P(p_{0})\right|\left|P(p_{1})\right|}n_{L}}\sum_{q_{0}\in P(p_{0})}\sum_{q_{1}\in P(p_{1})}\sum_{ii'}\Theta_{ii'}^{(L,q_{0}q_{1})}(x,y)\dot{\sigma}\left(\tilde{\alpha}_{i}^{(L,q_{0})}(x)\right)\dot{\sigma}\left(\tilde{\alpha}_{i'}^{(L,q_{1})}(y)\right)\\
 & \qquad\qquad\qquad\qquad\qquad\qquad\qquad\qquad\qquad\qquad\qquad\qquad\qquad W_{ij}^{(L,q_{0}\to p_{0})}W_{i'j'}^{(L,q_{1}\to p_{1})}\\
 & +\delta_{jj'}\beta^{2}+\delta_{jj'}\frac{1-\beta^{2}}{\sqrt{\left|P(p_{0})\right|\left|P(p_{1})\right|}n_{L}}\sum_{q_{0}\in P(p_{0})}\sum_{q_{1}\in P(p_{1})}\chi(q_{0}\to p_{0},q_{1}\to p_{1})\sum_{i}\alpha_{i}^{(L,q_{0})}(x)\alpha_{i}^{(L,q_{1})}(y)
\end{align*}
which, by assumption, converges as $n_{1}\to\infty$, $\ldots$,$n_{L-1}\to\infty$
to
\begin{align*}
 & \frac{1-\beta^{2}}{\sqrt{\left|P(p_{0})\right|\left|P(p_{1})\right|}n_{L}}\sum_{q_{0}\in P(p_{0})}\sum_{q_{1}\in P(p_{1})}\sum_{i}\Theta_{\infty}^{(L,q_{0}q_{1})}(x,y)\dot{\sigma}\left(\tilde{\alpha}_{i}^{(L,q_{0})}(x)\right)\dot{\sigma}\left(\tilde{\alpha}_{i}^{(L,q_{1})}(y)\right)\\
 & \qquad\qquad\qquad\qquad\qquad\qquad\qquad\qquad\qquad\qquad\qquad\qquad\qquad W_{ij}^{(L,q_{0}\to p_{0})}W_{ij'}^{(L,q_{1}\to p_{1})}\\
 & +\delta_{jj'}\beta^{2}+\delta_{jj'}\frac{1-\beta^{2}}{\sqrt{\left|P(p_{0})\right|\left|P(p_{1})\right|}n_{L}}\sum_{q_{0}\in P(p_{0})}\sum_{q_{1}\in P(p_{1})}\chi(q_{0}\to p_{0},q_{1}\to p_{1})\sum_{i}\alpha_{i}^{(L,q_{0})}(x)\alpha_{i}^{(L,q_{1})}(y).
\end{align*}
 As $n_{L}\to\infty$, using the previous results on the preactivations
and the law of large number, the NTK converges to 
\begin{align*}
 & \frac{1-\beta^{2}}{\sqrt{\left|P(p_{0})\right|\left|P(p_{1})\right|}}\sum_{q_{0}\in P(p_{0})}\sum_{q_{1}\in P(p_{1})}\Theta_{\infty}^{(L,q_{0}q_{1})}(x,y)\mathbb{L}_{\Sigma^{(L,q_{0}q_{1})}}^{\dot{\sigma}}\left(x,y\right)\mathbb{E}\left[W_{ij}^{(L,q_{0}\to p_{0})}W_{ij'}^{(L,q_{1}\to p_{1})}\right]\\
 & \qquad\qquad\qquad\qquad\qquad\qquad\qquad\qquad\qquad\qquad\qquad\qquad\qquad\\
 & +\delta_{jj'}\beta^{2}+\delta_{jj'}\frac{1-\beta^{2}}{\sqrt{\left|P(p_{0})\right|\left|P(p_{1})\right|}}\sum_{q_{0}\in P(p_{0})}\sum_{q_{1}\in P(p_{1})}\chi(q_{0}\to p_{0},q_{1}\to p_{1})\mathbb{L}_{\Sigma^{(L,q_{0}q_{1})}}^{\sigma}\left(x,y\right),
\end{align*}
which can be simplified--using the fact that $\mathbb{E}\left[W_{ij}^{(L,q_{0}\to p_{0})}W_{ij'}^{(L,q_{1}\to p_{1})}\right]=\chi(q_{0}\to p_{0},q_{1}\to p_{1})\delta_{jj'}$--into:

\begin{align*}
 & \delta_{jj'}\frac{1-\beta^{2}}{\sqrt{\left|P(p_{0})\right|\left|P(p_{1})\right|}}\sum_{q_{0}\in P(p_{0})}\sum_{q_{1}\in P(p_{1})}\chi(q_{0}\to p_{0},q_{1}\to p_{1})\Theta_{\infty}^{(L,q_{0}q_{1})}(x,y)\mathbb{L}_{\Sigma^{(L,q_{0}q_{1})}}^{\dot{\sigma}}\left(x,y\right)\\
 & +\delta_{jj'}\Sigma^{(L+1,p_{0}p_{1})}(x,y),
\end{align*}
which proves the assertions. 
\end{proof}

\section{DC-NN Order and Chaos\label{sec:DC-NN-Freeze-and-Chaos}}

In this section, in order to study the behaviour of DC-NNs in the
bulk and to avoid dealing with border effects, studied in Section
\ref{sec:Border-Effects-appendix}, we assume that for all layers
$\ell$ there is no border, i.e. the positions $p$ are in $\mathbb{Z}^{d}$.
Let us consider a DC-NN with up-sampling $s\in\{2,3,...\}^{d}$ where
the window sizes for all layers are all set equal to $\pi=\omega=\{0,...,w_{1}s_{1}-1\}\times...\times\{0,...,w_{d}s_{d}-1\}$.
A position $p$ has therefore $w_{1}\cdots w_{d}$ parents which are
given by 
\[
P(p)=\{\left\lfloor \nicefrac{p_{0}}{s_{0}}\right\rfloor ,\left\lfloor \nicefrac{p_{0}}{s_{0}}\right\rfloor +1,...,\left\lfloor \nicefrac{p_{0}}{s_{0}}\right\rfloor +w_{1}\}\times...\times\{\left\lfloor \nicefrac{p_{d}}{s_{d}}\right\rfloor ,\left\lfloor \nicefrac{p_{d}}{s_{d}}\right\rfloor +1,...,\left\lfloor \nicefrac{p_{d}}{s_{d}}\right\rfloor +w_{d}\}.
\]
Two connections $q\to p$ and $q'\to p'$ are shared if and only if
$s\mid p-p'$ (i.e. for any $i=1,...,d$, $s_{i}\mid p_{i}-p_{i}'$
) and $q_{i}-q_{i}'=\frac{p_{i}-p_{i}'}{s_{i}}$ for any $i=1,...,d$. 

Propositions \ref{prop:conv_Sigma_parents} and \ref{prop:convergence_NTK_parents}
hold true in this setting. By Proposition \ref{prop:Indep-annex},
if the nonlinearity $\sigma$ is standardized, $\Sigma^{(\ell,pp)}(x,x)=1$
for any $x\in\mathbb{S}_{n_{0}}^{I_{0}}$ and any $p\in I_{\ell}$.
The activation kernels $\Sigma^{(\ell,pp')}(x,y)$ for any two inputs
$x,y\in\mathbb{S}_{n_{0}}^{I_{0}}$ and two output positions $p,p'\in\mathbb{Z}^{d}$
are therefore defined recursively by: 
\begin{align*}
\Sigma^{(1,pp')}(x,y) & =\beta^{2}+\delta_{s|p-p'}\frac{1-\beta^{2}}{\left|P(p)\right|n_{0}}\sum_{q\in P(p)}\left(x_{q}\right)^{T}y_{q+\frac{p'-p}{s}},\\
\Sigma^{(\ell+1,pp')}(x,y) & =\beta^{2}+\delta_{s|p-p'}\frac{1-\beta^{2}}{\left|P(p)\right|}\sum_{q\in P(p)}R_{\sigma}\left(\Sigma^{(\ell,q,q+\frac{p'-p}{s})}(x,y)\right),
\end{align*}
where $\frac{p'-p}{s}=\left(\frac{p'_{i}-p_{i}}{s_{i}}\right)_{i}$
is a valid position since $s|p-p'$. Similarly, the NTK at initialization
satisfies the following recursion:

\[
\Theta_{\infty}^{(L+1,pp')}(x,y)=\Sigma^{(L+1,pp')}(x,y)+\delta_{s|p-p'}\frac{1-\beta^{2}}{\left|P(p)\right|}\sum_{q\in P(p)}\Theta_{\infty}^{(L,q,q+\frac{p'-p}{s})}(x,y)R_{\dot{\sigma}}\left(\Sigma^{(L,q,q+\frac{p'-p}{s})}(x,y)\right).
\]

\begin{rem*}
Recall that the $s$-valuation $v_{s}\left(n\right)$ of a number
$n\in\mathbb{Z}^{d}$ is the largest $k\in\left\{ 0,1,2,\ldots\right\} $
such that $s_{i}^{k}\mid n_{i}$ for all dimensions $i=1,...,d$.
For two pixels $p,p'\in\mathbb{Z}^{d}$ and any input vectors $x,y\in\mathbb{S}_{n_{0}}^{I_{0}}$,
if $v_{s}(p'-p)<\ell$ the activation kernel $\Sigma^{(\ell,pp')}(x,y)$
does not depend neither on $x$ nor on $y$. More precisely, if $v=v_{s}(p'-p)=0$,
we have
\[
\Sigma^{(\ell,pp')}(x,y)=\beta^{2},
\]
and for a general $v<\ell$:
\begin{alignat*}{1}
c_{v}:=\Sigma^{(\ell,pp')}(x,y) & =\left(B_{\beta}\circ R_{\sigma}\right)^{\circ v}(\beta^{2}).
\end{alignat*}
In particular, if $v<L$, the NTK is therefore also equal to a constant:
\[
\Theta_{\infty}^{(L,pp')}(x,y)=\sum_{k=0}^{v}c_{k}(1-\beta^{2})^{k}\prod_{m=0}^{k-1}R_{\dot{\sigma}}(c_{m}).
\]
\end{rem*}
We establish the bounds on the rate of convergence in the ``order''
region and on the values of the activations kernel in the chaos region
for DC-NNs. 
\begin{prop}
In the setting introduced above, for a standardized twice differentiable
$\sigma$, for $x,y\in\mathbb{S}_{n_{0}}^{I_{0}}$, and any positions
$p,p'\in I_{\ell}$, taking $k=\min\{v_{s}(p'-p),\ell\},$ we have:

If $r_{\sigma,\beta}<1$ then:
\[
1\geq\Sigma^{(\ell,pp')}(x,y)\geq1-2(1-\beta^{2})r_{\sigma,\beta}^{k}.
\]

If $r_{\sigma,\beta}>1$ then there exists a fixed point $a\in[0,1)$
of \textup{$B_{\beta}\circ R_{\sigma}$} such that: 
\begin{itemize}
\item If $k<\ell$: 
\[
\left|\Sigma^{(\ell,pp')}(x,y)\right|\leq\max\left\{ \beta^{2},a\right\} ,
\]
\item If $p'-p=ms^{\ell}$ and there is a $c\leq1$ such that for all input
positions $q\in P^{\circ\ell}(p)$, $\left|\frac{1}{n_{0}}x_{q}^{T}y_{q+m}\right|\leq c$,
then 
\[
\left|\Sigma^{(\ell,pp')}(x,y)\right|\leq\max\left\{ \beta^{2}+(1-\beta^{2})c,a\right\} .
\]
\end{itemize}
\end{prop}

\begin{proof}
Let us denote $r=r_{\sigma,\beta}$. Let us suppose that $r<1$ and
let us prove the first assertion by induction on $\ell$. If $\ell=1$,
then 
\begin{align*}
\Sigma^{(1,pp')}(x,y) & =\beta^{2}+\delta_{s|p-p'}\frac{1-\beta^{2}}{\left|P(p)\right|n_{0}}\sum_{q\in P(p)}\left(x_{q}\right)^{T}y_{q+\frac{p'-p}{s}}\\
 & \geq\beta^{2}-\delta_{s|p-p'}(1-\beta^{2})\\
 & \geq1-2(1-\beta^{2})
\end{align*}
For the induction step, if we suppose that the inequality holds true
for $\ell$, then 
\begin{align*}
\Sigma^{(\ell+1,pp')}(x,y) & \geq\beta^{2}+\delta_{s|p-p'}\frac{1-\beta^{2}}{\left|P(p)\right|}\sum_{q=0}^{\nicefrac{w}{s}}R_{\sigma}\left(1-2(1-\beta^{2})r^{k-1}\right)\\
 & \geq\beta^{2}+\delta_{s|p-p'}\frac{1-\beta^{2}}{\left|P(p)\right|}\sum_{q=0}^{\nicefrac{w}{s}}1-2(1-\beta^{2})R_{\dot{\sigma}}(1)r^{k-1}\\
 & \geq\beta^{2}+\delta_{s|p-p'}\left(1-\beta^{2}-2(1-\beta^{2})r^{k}\right)\\
 & =\begin{cases}
1-(1-\beta^{2}) & \text{if }k=0\\
1-2(1-\beta^{2})r^{k} & \text{if }k>0
\end{cases}\\
 & \geq1-2(1-\beta{}^{2})r^{k}
\end{align*}
Now let us suppose that $r>1$. If $k<\ell$, then $\left|\Sigma^{(\ell,pp')}(x,y)\right|=\left|\left(B_{\beta}\circ R_{\sigma}\right)^{\circ k}\left(\beta^{2}\right)\right|<\max\left\{ \beta^{2},a\right\} .$
Let us suppose at last that $k=\ell$ and let us prove the last assertion
by induction on $\ell$. If $\ell=1$, then 
\begin{align*}
\left|\Sigma^{(1,pp')}(x,y)\right| & \leq\beta^{2}+\frac{1-\beta^{2}}{\left|P(p)\right|n_{0}}\sum_{q\in P(p)}\left|x_{q}^{T}y_{q+\frac{p'-p}{s}}^{T}\right|\\
 & \leq\beta^{2}+\frac{1-\beta^{2}}{\left|P(p)\right|}\sum_{q\in P(p)}c\\
 & =\beta^{2}+(1-\beta^{2})c.
\end{align*}

For the induction step, if we suppose that the inequality holds true
for $\ell$, then 
\begin{align*}
\left|\Sigma^{(\ell+1,pp')}(x,y)\right| & \leq\beta^{2}+\frac{(1-\beta^{2})}{\left|P(p)\right|}\sum_{q\in P(p)}\left|R_{\sigma}\left(\Sigma^{(\ell,q,q+\frac{p'-p}{s})}(x,y)\right)\right|\\
 & \leq\beta^{2}+\frac{(1-\beta^{2})}{\left|P(p)\right|}\sum_{q\in P(p)}R_{\sigma}\left(\max\{\beta^{2}+(1-\beta^{2})c,a\}\right)\\
 & =B_{\beta}\circ R_{\sigma}\left(\max\{\beta^{2}+(1-\beta^{2})c,a\}\right)\\
 & \leq\max\{\beta^{2}+(1-\beta^{2})c,a\},
\end{align*}
 which allow us to conclude. 
\end{proof}
The NTK features the same two regimes:
\begin{thm}
\label{prop:freeze_chaos_DC-NN}Take $I_{0}=\mathbb{Z}^{d}$, and
consider a DC-NN with upsampling stride $s\in\left\{ 2,3,\ldots\right\} ^{d}$,
windows $\pi_{\ell}=\omega_{\ell}=\left\{ 0,\ldots,w_{1}s_{1}-1\right\} \times...\times\left\{ 0,\ldots,w_{d}s_{d}-1\right\} $
for $w\in\left\{ 1,2,3,\ldots\right\} ^{d}$. For a standardized twice
differentiable $\sigma$, there exist constants $C_{1},C_{2}>0,$
such that the following holds: for $x,y\in\mathbb{S}_{n_{0}}^{I_{0}}$,
and any positions $p,p'\in I_{L}$, we have:

\textbf{Order:} When $r_{\sigma,\beta}<1$, taking $v=\min\left(v_{s}\left(p-p'\right),L-1\right)$,
taking $v=L-1$ if $p=p'$ and $r=r_{\sigma,\beta}$, we have 
\[
\frac{1-r^{v+1}}{1-r^{L}}-C_{1}(v+1)r^{v}\leq\vartheta_{\infty}^{\left(L,p,p'\right)}\left(x,y\right)\leq\frac{1-r^{v+1}}{1-r^{L}}.
\]
\textbf{Chaos:} When $r_{\sigma,\beta}>1$, if either $v_{s}\left(p-p'\right)<L$
or if there exists a $c<1$ such that for all positions $q\in I_{0}$
which are ancestor of $p$, $\left|x_{q}^{T}y_{q+\frac{p'-p}{s^{L}}}\right|<c$,
then there exists $h<1$ such that
\[
\left|\vartheta_{\infty}^{\left(L,p,p'\right)}\left(x,y\right)\right|\leq C_{2}h^{L}.
\]
\end{thm}

\begin{proof}
Let us denote $r=r_{\sigma,\beta}$ and let us suppose that $r<1$.
The NTK can be bounded recursively
\begin{align*}
\Theta_{\infty}^{(L,pp')}(x,y) & =\Sigma^{(L,pp')}(x,y)+\delta_{s|p-p'}\frac{1-\beta^{2}}{\left|P(p)\right|}\sum_{q\in P(p)}\Theta_{\infty}^{(L-1;q,q+\frac{p'-p}{s})}(x,y)R_{\dot{\sigma}}\left(\Sigma^{(L-1;q,q+\frac{p'-p}{s})}(x,y)\right)\\
 & \geq1-2(1-\beta^{2})r^{v}+\delta_{s|p-p'}\frac{1}{\left|P(p)\right|}\sum_{q\in P(p)}\Theta_{\infty}^{(L;q,q+\frac{p'-p}{s})}(x,y)\left(r-\psi2(1-\beta^{2})^{2}r^{v-1}\right).
\end{align*}
Unrolling this inequality, we get: 
\begin{align*}
\Theta_{\infty}^{(L,pp')}(x,y) & =\sum_{k=0}^{v}\left(1-2(1-\beta{}^{2})r^{k}\right)\prod_{m=k+1}^{v}\left(r-\psi2(1-\beta^{2})^{2}r^{m-1}\right)\\
 & \geq\sum_{k=0}^{v}r^{v-k}-2(1-\beta{}^{2})r^{v-k}r^{k}-\psi2(1-\beta^{2})^{2}\sum_{m=k+1}^{v}r^{v-k-1}r^{m-1}\\
 & =\frac{1-r^{v+1}}{1-r}-2(1-\beta{}^{2})(v+1)r^{v}-\psi2(1-\beta^{2})^{2}\sum_{k=0}^{v}r^{v-1}\sum_{m=0}^{v-k-1}r^{m}\\
 & \geq\frac{1-r^{v+1}}{1-r}-2(1-\beta{}^{2})\left[r+\frac{\psi(1-\beta^{2})}{1-r}\right](v+1)r^{v-1}\\
 & \geq\frac{1-r^{v+1}}{1-r}-C_{\sigma,\beta}(v+1)r^{v}.
\end{align*}
For the upper bound, we have: $\Theta_{\infty}^{(L,pp')}(x,y)\leq\sum_{\ell=L-k}^{L}1\prod_{m=\ell+1}^{L}r=\frac{1-r^{v+1}}{1-r}.$
Thus, we get the same bounds as in the FC-NNs case, but with respect
to $v$, which is the maximal integer strictly smaller than $L$ such
that $s^{v}|p-p'$:
\[
\frac{1-r^{v+1}}{1-r}\geq\Theta_{\infty}^{(L,pp')}(x,y)\geq\frac{1-r^{v+1}}{1-r}-C(v+1)r^{v}.
\]
Dividing by $\Theta_{\infty}^{(L,pp)}(x,x)$ which is bounded in the
ordered regime (see proof of Proposition \ref{prop:Indep-annex})
as $L\to\infty$, one gets the desired result. 

If $r>1$, there are two cases. When $p'-p=ks^{L}$ then if there
exists $c<1$ such that $\left|x_{q}^{T}y_{q+k}\right|<cn_{0}$ for
all ancestors $q$ of $p$. Writing $z=\max\{\beta^{2}+(1-\beta^{2})c,a\}$
and $w=(1-\beta^{2})R_{\dot{\sigma}}(z)<r$ such that $\left|\Sigma^{(\ell;q,q+ks^{\ell})}(x,y)\right|<z$
for all position $q$ at layer $\ell$ which is an ancestor of $p$.
Then
\[
\left|\Theta_{\infty}^{(L,pp')}(x,y)\right|\leq\sum_{\ell=1}^{L}vw^{L-\ell}=v\frac{1-w^{L}}{1-w}
\]
such that 
\[
\frac{\left|\Theta_{\infty}^{(L,pp')}(x,y)\right|}{\left|\Theta_{\infty}^{(L,pp)}(x,x)\right|}\leq c\frac{1-r}{1-w}\frac{1-w^{L}}{1-r^{L}}\leq C(\sigma,\beta)\left(\frac{w}{r}\right)^{L}
\]
which goes to zero exponentially.

If $p'-p$ is not divisible by $s^{L}$ then for $z=\max\{\beta^{2},a\}$
and $w=(1-\beta^{2})R_{\dot{\sigma}}(z)<r$ 
\[
\left|\Theta_{\infty}^{(L,pp')}(x,y)\right|\leq\sum_{\ell=L-v+1}^{L}zw^{L-\ell}=z\frac{1-w^{v}}{1-w}
\]
which also converges exponentially to 0.
\end{proof}

\subsection{Adapting the learning rate}

Let us suppose that we multiply the learning rate of the $\ell$-th
layer weights and bias by $S^{-\frac{\ell}{2}}$ where $S=\prod_{i}s_{i}$.
This is slightly different than what we propose in the main, where
the learning rate of the bias are multiplied by $S^{-\frac{\ell+1}{2}}$
instead of $S^{-\frac{\ell}{2}}$, but it greatly simplifies the formulas.
Furthermore, the balance between the weights and bias can be modified
with the meta-parameter $\mathbf{\beta}$ to achieve a similar result.
The NTK then takes the value:
\begin{align*}
\Theta^{(L,pp)}(x,x) & =\sum_{\ell=1}^{L}S^{-\frac{\ell}{2}}\prod_{n=\ell+1}^{L}r\\
 & =\sum_{\ell=1}^{L}S^{-\frac{\ell}{2}}r^{L-\ell}\\
 & =\sum_{\ell=0}^{L-1}S^{-\frac{L-\ell}{2}}r^{\ell}\\
 & =S^{-\frac{L}{2}}\frac{1-\left(\sqrt{S}r\right)^{L}}{1-\sqrt{S}r}
\end{align*}
This leads to another transtion inside the ``order'' regime: if
$\sqrt{S}r<1$ the NTK $\Theta_{\infty}^{(L,pp)}(x,x)$ goes to zero
and if $\frac{1}{\sqrt{S}}<r<1$ it converges to a constant. If we
translate the bound of Proposition \ref{prop:freeze_chaos_DC-NN}
to the NTK with varying learning rates, the convergence to a constant
is only guaranteed when $\sqrt{S}r<1$, which suggests that adapting
the learning (or changing the number of channels) does reduce the
checkerboard artifacts (as confirmed by numerical experiments):
\begin{prop}
If $r<1$ the limiting NTK at any two inputs $x,y$ such that for
all $p\in\mathbb{Z}$, $\left\Vert x^{p}\right\Vert =\left\Vert y^{p}\right\Vert =\sqrt{n_{0}}$
and for any two output positions $p$ and $p'$, such that $k$ is
the maximal integer in $\{0,...,L-1\}$ such that $s^{k}$ divides
the difference $p-p'$ then:
\[
\frac{1-(\sqrt{S}r)^{k+1}}{1-(\sqrt{S}r)^{L}}\geq\vartheta_{\infty}^{(L,pp')}(x,y)\geq\frac{1-(\sqrt{S}r)^{k+1}}{1-(\sqrt{S}r)^{L}}-\frac{C_{\sigma,\beta}(\sqrt{S}r)^{k}}{\left|1-(\sqrt{S}r)^{L}\right|}
\]
\end{prop}

\begin{proof}
The NTK can be bounded recursively
\begin{align*}
\Theta_{\infty}^{(L,pp')}(x,y) & =S^{-\nicefrac{L-1}{2}}\Sigma^{(L,pp')}(x,y)+\delta_{s|p-p'}\frac{1-\beta^{2}}{\left|P(p)\right|}\sum_{q\in P(p)}\Theta_{\infty}^{(L-1;q,q+\frac{p'-p}{s})}(x,y)R_{\dot{\sigma}}\left(\Sigma^{(L-1;q,q+\frac{p'-p}{s})}(x,y)\right)\\
 & \geq S^{-\nicefrac{L-1}{2}}(1-2(1-\beta^{2})r^{k})+\delta_{s|p-p'}\frac{1}{\left|P(p)\right|}\sum_{q\in P(p)}\Theta_{\infty}^{(L;q,q+\frac{p'-p}{s})}(x,y)\left(r-\psi2(1-\beta^{2})^{2}r^{k-1}\right)
\end{align*}
unrolling, we get
\begin{align*}
 & \Theta_{\infty}^{(L,pp')}(x,y)\\
 & \geq\sum_{m=0}^{k}S^{-\frac{L-k+m}{2}}\left(1-2(1-\beta{}^{2})r^{m}\right)\prod_{n=m+1}^{k}\left(r-\psi2(1-\beta^{2})^{2}r^{n-1}\right)\\
 & \geq\sum_{m=0}^{k}S^{\frac{k-m-L}{2}}r^{k-m}-S^{\frac{k-m-L}{2}}2(1-\beta{}^{2})r^{k-m}r^{m}-S^{\frac{k-m-L}{2}}\psi2(1-\beta^{2})^{2}\sum_{n=m+1}^{k}r^{k-m-1}r^{n-1}\\
 & \geq S^{-\nicefrac{L}{2}}\frac{1-(\sqrt{S}r)^{k+1}}{1-\sqrt{S}r}-2\frac{1-\beta^{2}}{1-S^{-\nicefrac{1}{2}}}S^{\frac{k-L}{2}}r^{k}-\psi2(1-\beta^{2})^{2}r^{k-1}\sum_{m=0}^{k}S^{\frac{k-m-L}{2}}\sum_{n=0}^{k-m-1}r^{n}\\
 & \geq S^{-\nicefrac{L}{2}}\frac{1-(\sqrt{S}r)^{k+1}}{1-\sqrt{S}r}-2\frac{1-\beta^{2}}{1-S^{-\nicefrac{1}{2}}}S^{\frac{k-L}{2}}r^{k}-\psi2(1-\beta^{2})^{2}r^{k-1}S^{\frac{k-L}{2}}\frac{1}{1-S^{-\nicefrac{1}{2}}}\frac{1}{1-r}\\
 & \geq S^{-\nicefrac{L}{2}}\frac{1-(\sqrt{S}r)^{k+1}}{1-\sqrt{S}r}-2\frac{1-\beta^{2}}{1-S^{-\nicefrac{1}{2}}}\left[1+\frac{\psi r(1-\beta^{2})}{1-r}\right]r^{k}S^{\frac{k-L}{2}}\\
 & \geq S^{-\nicefrac{L}{2}}\left(\frac{1-(\sqrt{S}r)^{k+1}}{1-\sqrt{S}r}-C_{\sigma,\beta}\left(\sqrt{S}r\right){}^{k}\right)
\end{align*}

and for the upper bound:
\[
\Theta_{\infty}^{(L,pp')}(x,y)\leq\sum_{m=0}^{k}S^{-\frac{L-k+m}{2}}\prod_{n=m+1}^{k}r=S^{-\nicefrac{L}{2}}\frac{1-(\sqrt{S}r)^{k+1}}{1-\sqrt{S}r}.
\]
Dividing by $\Theta_{\infty}^{(L,pp)}(x,x)$ we obtain
\[
\frac{1-(\sqrt{S}r)^{k+1}}{1-(\sqrt{S}r)^{L}}\geq\vartheta_{\infty}^{(L,pp')}(x,y)\geq\frac{1-(\sqrt{S}r)^{k+1}}{1-(\sqrt{S}r)^{L}}-\frac{C_{\sigma,\beta}(\sqrt{S}r)^{k}}{\left|1-(\sqrt{S}r)^{L}\right|}
\]
\end{proof}

\section{Border Effects\label{sec:Border-Effects-appendix}}

With the usual scaling of $\frac{1}{\sqrt{\nicefrac{\left|\omega\right|}{s_{1}\ldots s_{d}}}}$,
in a General ConvNet, the positions on the border have less parents
and hence a lower activation variance. In this section, we show, in
a special example, how this parametrization leads to border effects
in the limiting activation kernels and NTK. This could be generalized
to a more general setting, yet, our main purpose is to show that with
the parent-based parametrization--as defined in Section \ref{sec:General-Convolutional-Networks}--no
border artifact is present in both kernels in this general setting.

The following proposition illustrates the border artifact present
in the usual NTK-parametrization. Let us consider a DC-NN with a standardized
ReLU nonlinearity, with $I_{0}=I_{1}\ldots=\mathbb{N}$, with up-sampling
stride of $2$, and windows $\pi_{0}=\omega_{0}=\pi_{1}=\omega_{1}=\ldots=\left\{ -3,-2,-1,0\right\} $.
In particular, there is only one border at position $0$. Using the
formalism of Section \ref{sec:General-Convolutional-Networks}, the
set of parents of a position $p$ is $P(p)=\{\left\lfloor \nicefrac{p}{2}\right\rfloor -1,\left\lfloor \nicefrac{p}{2}\right\rfloor \}\cap\mathbb{N}$.
In particular, any generic position in any hidden or last layer has
$2$ parents except for the border $p=0$ for which $P(0)=\{0\}$. 
\begin{prop}
In the setting introduced above, for any $x\in\mathbb{S}_{n_{0}}^{I_{0}}$,
the kernels satisfy: 
\[
\Sigma^{(\ell,00)}(x,x)=\frac{\beta^{2}+\left(\nicefrac{r}{2}\right)^{\ell+1}}{1-\nicefrac{r}{2}}\text{ and }\Theta_{\infty}^{(L,00)}(x,x)=\frac{\beta^{2}(1-\left(\nicefrac{r}{2}\right)^{L})}{\left(1-\nicefrac{r}{2}\right)^{2}}+L\frac{\left(\nicefrac{r}{2}\right)^{L+1}}{1-\nicefrac{r}{2}}.
\]
 In particular $\Sigma^{(\ell,00)}(x,x)$ is smaller than the ``bulk-value''
$\lim_{p\to\infty}\Sigma^{(\ell,pp)}(x,x)=1$ and $\Theta_{\infty}^{(L,00)}(x,x)$
is smaller than the ``bulk-value'' $\lim_{p\to\infty}\Theta_{\infty}^{(L,pp)}(x,x)=\frac{1-r^{L}}{1-r}$. 
\end{prop}

\begin{proof}
Recall that for the standardized ReLU, $r_{\sigma,\beta}=1-\beta^{2}$.
From now on, we denote $r=r_{\sigma,\beta}$ and $x$ is an element
of $\mathbb{S}_{n_{0}}^{I_{0}}$. For any $\ell=0,1\ldots$, we have:
\[
\Sigma^{(\ell+1,00)}(x,x)=\beta^{2}+\frac{1-\beta^{2}}{2}\sum_{q\in P(0)}\mathbb{E}_{z\sim\mathcal{N}(0,\Sigma_{qq}^{(\ell)}(x,x))}\left[\sigma(x)^{2}\right]=\beta^{2}+\frac{1-\beta^{2}}{2}\Sigma^{(\ell,00)}(x,x).
\]
Since $x\in\mathbb{S}_{n_{0}}^{I_{0}}$, we get $\Sigma^{(1)}(x,x)=\beta^{2}+\frac{r}{2}$:
this implies the following equalities:
\begin{align*}
\Sigma^{(\ell,00)}(x,x) & =\left(\nicefrac{r}{2}\right)^{\ell}+\sum_{k=0}^{\ell-1}\beta^{2}\left(\nicefrac{r}{2}\right)^{k}\\
 & =\left(\nicefrac{r}{2}\right)^{\ell}+\beta^{2}\frac{1-\left(\nicefrac{r}{2}\right)^{\ell}}{1-\nicefrac{r}{2}}\\
 & =\frac{\beta^{2}}{1-\nicefrac{r}{2}}+\frac{\left(\nicefrac{r}{2}\right)^{\ell}-\left(\nicefrac{r}{2}\right)^{\ell+1}-\beta^{2}\left(\nicefrac{r}{2}\right)^{\ell}}{1-\nicefrac{r}{2}}\\
 & =\frac{\beta^{2}+\left(\nicefrac{r}{2}\right)^{\ell+1}}{1-\nicefrac{r}{2}}.
\end{align*}
For the limiting NTK, with the usual NTK parametrization, the following
recursion holds:
\[
\Theta_{\infty}^{(L+1,00)}(x,x)=\Sigma^{(L+1,00)}(x,x)+\frac{r}{2}\Theta_{\infty}^{(L,00)}(x,x)\mathbb{L^{\dot{\sigma}}}_{\Sigma^{(L,00)}}(x,x).
\]
Note that for the standardized ReLU, $\dot{\sigma}$ is a rescaled
Heaviside, thus $\mathbb{L^{\dot{\sigma}}}_{\Sigma^{(L,00)}}(x,x)=\mathbb{E}_{x\sim\mathcal{N}(0,\Sigma^{(L,00)}(x,x))}\left[\dot{\sigma}(x)^{2}\right]=2\mathbb{E}_{x\sim\mathcal{N}(0,1)}[\mathbb{I}_{x\geq0}]=1$.
This implies: 
\begin{align*}
\Theta^{(L,00)}(x,x) & =\sum_{\ell=1}^{L}\Sigma^{(\ell,00)}(x,x)\left(\nicefrac{r}{2}\right)^{L-\ell}\\
 & =\sum_{\ell=1}^{L}\left(\frac{\beta^{2}}{1-\nicefrac{r}{2}}+\frac{\left(\nicefrac{r}{2}\right)^{\ell+1}}{1-\nicefrac{r}{2}}\right)\left(\nicefrac{r}{2}\right)^{L-\ell}\\
 & =\frac{\beta^{2}}{1-\nicefrac{r}{2}}\sum_{\ell=1}^{L}\left(\nicefrac{r}{2}\right)^{L-\ell}+L\frac{\left(\nicefrac{r}{2}\right)^{L+1}}{1-\nicefrac{r}{2}}\\
 & =\frac{\beta^{2}(1-\left(\nicefrac{r}{2}\right)^{L})}{\left(1-\nicefrac{r}{2}\right)^{2}}+L\frac{\left(\nicefrac{r}{2}\right)^{L+1}}{1-\nicefrac{r}{2}}.
\end{align*}

The ``bulk-values'' for the activation kernels and the limiting
NTK kernel can be deduced from the proof of Proposition \ref{prop:Indep-annex}.
A tedious study of variation of functions allows to prove the assertion
on the boundary/bulk comparison. 
\end{proof}
As a consequence of the previous proposition, in the limits as $\ell$
and $L$ goes to infinity, the ratio boundary/bulk value is bounded
by $\max\left(1,c\beta^{2}\right)$: the smaller $\beta$ is, the
stronger the boundary effect will be. 

In the parent-based parametrization, the variance of the neurons throughout
the network is always equal to $1$ and the NTK $\Theta_{\infty,pp}^{(L)}(x,x)$
becomes independent of the position $p$: the border artifacts disappear. 
\begin{prop}
\label{prop:Indep-annex}For the parent-based parametrization of DC-NNs,
if the nonlinearity is standardized, $\left(\Sigma^{\left(L\right)}\right)_{pp}\left(x\right)$
and $\left(\Theta_{\infty}^{\left(L\right)}\right)_{pp}\left(x\right)$
do not depend neither on $p\in I_{L}$ nor on $x\in\mathbb{S}_{n_{0}}^{I_{0}}$.
\end{prop}

\begin{proof}
Actually, we will prove the stronger statement: for any General Convolutional
Network, as defined in Section \ref{sec:General-Convolutional-Networks},
for any standardized nonlinearity, for any $x\in\mathbb{S}_{n_{0}}^{I_{0}}$
and any $p\in I_{L}$,
\[
\Sigma^{(L,pp)}(x,x)=1,\quad\text{and }\Theta_{\infty}^{(L,pp)}(x,x)=\frac{1-r^{L}}{1-r}.
\]
For the activation kernels,this is proven by induction on $\ell$
. For any $x\in\mathbb{S}_{n_{0}}^{I_{0}}$ and any $p\in I_{1}$:

\begin{align*}
\Sigma^{(1,pp)}(x,x) & =\beta^{2}+\frac{1-\beta^{2}}{\left|P(p)\right|n_{0}}\sum_{q\in P(p)}\sum_{q'\in P(p)}\chi(q\to p,q'\to p)x_{q}^{T}x_{q'}\\
 & =\beta^{2}+\frac{1-\beta^{2}}{\left|P(p)\right|n_{0}}\sum_{q\in P(p)}x_{q}^{T}x_{q}\\
 & =\beta^{2}+(1-\beta^{2})\\
 & =1,
\end{align*}
 and if the assertion holds true for $L$, then:

\begin{align*}
\Sigma^{(L+1,pp)}(x,x) & =\beta^{2}+\frac{1-\beta^{2}}{\left|P(p)\right|n_{0}}\sum_{q\in P(p)}\sum_{q'\in P(p)}\chi(q\to p,q'\to p)\Sigma^{(L,qq')}(x,x)\\
 & =\beta^{2}+\frac{1-\beta^{2}}{\left|P(p)\right|n_{0}}\sum_{q\in P(p)}\Sigma^{(L,qq)}(x,x)\\
 & =1.
\end{align*}

For the activation kernels, this is proven by induction on $L$. It
is easy to see that $\Theta_{\infty}^{(1,pp)}(x,x)=1$ is valid for
any $x\in\mathbb{S}_{n_{0}}^{I_{0}}$ and any $p\in I_{L}$. Let us
show the induction step:
\begin{align*}
\Theta_{\infty}^{(L+1,pp)}(x,x) & =\Sigma^{(L+1,pp)}(x,x)+\frac{1-\beta^{2}}{\left|P(p)\right|}\sum_{q\in P(p)}\Theta_{\infty}^{(L,qq)}(x,x)R_{\dot{\sigma}}\left(\Sigma^{(L,qq)}(x,x)\right)\\
 & =1+r\Theta_{\infty}^{(L,qq)}(x,x).
\end{align*}
Thus, $\Theta_{\infty}^{(L,pp)}(x,x)=\sum_{\ell=1}^{L}r^{L-\ell}=\frac{1-r^{L}}{1-r}.$
\end{proof}

\section{Layerwise Contributions to the NTK and Checkerboard Patterns}

In a DC-NN with stride $s\in\{2,3,...\}^{d}$, if two connection weight
matrices $W^{(\ell,q\to p)}$ and $W^{(\ell,q'\to p')}$ are shared
then $s\mid p'-p$. In other words, $\chi(q\to p,q'\to p')=0$ as
soon as $s\nmid p'-p$. The limiting contribution of the weights $\Theta_{\infty}^{(L:W^{(\ell)})}$
and bias $\Theta_{\infty}^{(L:b^{(\ell)})}$ to the limiting NTK can
be formulated recursively. For the last layer $L-1$ we have
\begin{align*}
\Theta_{\infty}^{(L:b^{(L-1)},pp')} & =\beta^{2}\\
\Theta_{\infty}^{(1:W^{(0)},pp')} & =\delta_{s|p-p'}\frac{1-\beta^{2}}{\left|P(p)\right|n_{0}}\sum_{q\in P(p)}x_{q}^{T}y_{q+\frac{p'-p}{s}}\\
\Theta_{\infty}^{(L:W^{(L-1)},pp')} & =\delta_{s|p-p'}\frac{1-\beta^{2}}{\left|P(p)\right|}\sum_{q\in P(p)}R_{\sigma}\left(\Sigma^{(L-1,q,q+\frac{p'-p}{s})}(x,y)\right)\text{ for \ensuremath{L>1}}
\end{align*}
and for the other layers, we have
\begin{align*}
\Theta_{\infty}^{(L+1:b^{(\ell)},pp')} & =\delta_{s|p-p'}\frac{1-\beta^{2}}{\left|P(p)\right|}\sum_{q\in P(p)}\Theta_{\infty}^{(L;b^{(\ell)},q,q+\frac{p'-p}{s})}(x,y)R_{\dot{\sigma}}\left(\Sigma^{(L,q,q+\frac{p'-p}{s})}(x,y)\right)\\
\Theta_{\infty}^{(L+1:W^{(\ell)},pp')} & =\delta_{s|p-p'}\frac{1-\beta^{2}}{\left|P(p)\right|}\sum_{q\in P(p)}\Theta_{\infty}^{(L;W^{(\ell)},q,q+\frac{p'-p}{s})}(x,y)R_{\dot{\sigma}}\left(\Sigma^{(L,q,q+\frac{p'-p}{s})}(x,y)\right).
\end{align*}

\begin{prop}
In a DC-NN with stride $s\in\{2,3,...\}^{d}$, we have $\Theta_{\infty}^{(L:W^{(\ell)},pp')}(x,y)=0$
if $s^{L-\ell}\nmid p'-p$ and $\Theta_{\infty}^{(L:b^{(\ell)},pp')}(x,y)=0$
if $s^{L-\ell-1}\nmid p'-p$.
\end{prop}

\begin{proof}
From the formulas of the limiting contributions $\Theta^{(L:W^{(\ell)})}$
and $\Theta^{(L:b^{(\ell)})}$, we see that the bias of the last layer
contribute to all pairs $p,p'$ while the bias only contribute to
pairs such that $s\mid p'-p$. Now by induction on $L$, if $\Theta^{(L:b^{(\ell)},qq')}$
and $\Theta^{(L:W^{(\ell)},qq')}$ only contribute to pairs $q,q'$
such that $s^{L-\ell-1}\mid q'-q$ and $s^{L-\ell}\mid q'-q$ then
\begin{align*}
\Theta_{\infty}^{(L+1:b^{(\ell)},pp')} & =\delta_{s|p-p'}\frac{1-\beta^{2}}{\left|P(p)\right|}\sum_{q\in P(p)}\Theta_{\infty}^{(L;b^{(\ell)},q,q+\frac{p'-p}{s})}(x,y)R_{\dot{\sigma}}\left(\Sigma^{(L,q,q+\frac{p'-p}{s})}(x,y)\right)\\
\Theta_{\infty}^{(L+1:W^{(\ell)},pp')} & =\delta_{s|p-p'}\frac{1-\beta^{2}}{\left|P(p)\right|}\sum_{q\in P(p)}\Theta_{\infty}^{(L;W^{(\ell)},q,q+\frac{p'-p}{s})}(x,y)R_{\dot{\sigma}}\left(\Sigma^{(L,q,q+\frac{p'-p}{s})}(x,y)\right)
\end{align*}
only contribute to pairs $p',p$ such that $s^{L-\ell}\mid p'-p$
and $s^{L+1-\ell}\mid p'-p$ as needed.
\end{proof}
\end{document}